\documentclass[11pt]{article}
\usepackage{enumerate}
\usepackage{pdfsync}
\usepackage[OT1]{fontenc}

\usepackage[usenames]{color}
\usepackage{smile}
\usepackage[colorlinks,
            linkcolor=red,
            anchorcolor=blue,
            citecolor=blue
            ]{hyperref}
\usepackage{fullpage}
\usepackage{hyperref}
\usepackage[protrusion=true,expansion=true]{microtype}
\usepackage{pbox}
\usepackage{setspace}
\usepackage{tabularx}
\usepackage{float}
\usepackage{wrapfig,lipsum}
\usepackage{enumitem}
\usepackage{microtype}
\usepackage{graphicx}
\usepackage{subfigure}
\usepackage{enumerate}
\usepackage{enumitem}
\usepackage{pgfplots}
\usetikzlibrary{arrows,shapes,snakes,automata,backgrounds,petri}
\usepackage{booktabs} 

\newcommand{\regret}{\mathrm{Regret}}

\newcommand{\KL}{\operatorname{KL}}
\newcommand{\kl}[2]{\ensuremath{{\mathsf{KL}}\left(#1\|#2\right)}}
\newcommand{\polylog}{{\mathrm{polylog}}}
\allowdisplaybreaks
\usepackage{colortbl}
\definecolor{LightCyan}{rgb}{0.8, 0.9, 1}
\definecolor{LightGray}{gray}{0.9}

\def \alg {\mathtt{Alg}}
\def \algname {\text{SMOSS}}

\usepackage{enumitem}
\usepackage{colortbl}
\definecolor{LightCyan}{rgb}{0.8, 0.9, 1}

\usepackage{xcolor}

\ifdefined\final
\usepackage[disable]{todonotes}
\else
\usepackage[textsize=tiny]{todonotes}
\fi
\newcommand{\aopt}{{\cA^{\text{opt}}}}
\newcommand{\asubopt}{{\cA^{\text{sub}}}}
\newcommand{\eventmiss}{{\cE_{\text{miss}}}}
\newcommand{\eventsubopt}{{\cE_{\text{sub}}}}
\newcommand{\eventopt}{{\cE_{\text{opt}}}}
\newcommand{\class}[2]{{\Xi}_{#1, #2}}

\makeatletter
\newcommand*{\rom}[1]{\expandafter\@slowromancap\romannumeral #1@}
\makeatother

\title{\huge Bandits with Multiple Optimal Arms: Minimax Regret and Non-Adaptivity}

\author{
    Kaixuan Ji\thanks{Equal contribution} \thanks{Department of Computer Science, University of California, Los Angeles, CA 90095, USA; e-mail: {\tt kaixuanji@cs.ucla.edu}} 
    ~~
    Qiwei Di\footnotemark[1] \thanks{Department of Computer Science, University of California, Los Angeles, CA 90095, USA; e-mail: {\tt qiwei2000@cs.ucla.edu}}
    ~~
    Qingyue Zhao\footnotemark[1] \thanks{Department of Computer Science, University of California, Los Angeles, CA 90095, USA; e-mail: {\tt zhaoqy24@cs.ucla.edu}}
    ~~
    Heyang Zhao\thanks{Department of Computer Science, University of California, Los Angeles, CA 90095, USA; e-mail: {\tt hyzhao@cs.ucla.edu}} 
    ~~
    Quanquan Gu\thanks{Department of Computer Science, University of California, Los Angeles, CA 90095, USA; e-mail: {\tt qgu@cs.ucla.edu}}
}
\date{}
\begin{document}
\maketitle

\begin{abstract}

We study multi-armed bandits (MAB) with multiple optimal arms, motivated by the fact that many practical decision making problems admit multiple correct answers. 
For $K$-armed bandits with $A$ optimal arms, we first provide a sharper analysis of previous sub-sampling algorithms~\citep{de2021bandits,zhu2020regret}, establishing a $\tilde{O}\Big(\frac{K-A}{\sqrt{KA}}\sqrt{T} \Big)$ minimax regret, where $T$ is the total number of interactions and $\tilde O(\cdot)$ drops all constant and logarithmic factors, improving the previous $\tilde{O}(\sqrt{KT/A})$ regret. We then provide a matching lower bound up to logarithmic factors, indicating that our established rate is nearly minimax-optimal.
We further show that the knowledge of $A$ up to $\tilde{O}(1)$ factors is necessary to achieve near-optimal regret, as near-optimal algorithms for one number of optimal arms must incur substantially larger regret than optimal regret for a smaller number.
Overall, our results provide a comprehensive minimax characterization of $K$-armed bandits with $A$ over the entire range of $1 \leq A \leq K-1$.




\end{abstract}

\section{Introduction}

Multi-armed bandits (MAB) provide a basic model of sequential decision making under uncertainty, whose implication covers a wide range of practical problems~\citep{bouneffouf2020survey,pryzant2023automatic,panda2025adaptive}. In an MAB, a learner repeatedly selects an action from an action set, observes a noisy reward, and aims at maximizing the cumulative reward via balancing exploration and exploitation. Classical results have characterized the difficulty of the problem by a worst case regret $\Theta(\sqrt{KT})$ when $T \geq K$, where $K$ is the number of arms and $T$ is the total number of interactions~\citep{auer2002finite,audibert2009minimax,lattimore2018refining}.


However, in practice, many decision making problems, including code generation~\citep{jimenez2024swe,jain2025livecodebench}, math problem solving~\citep{hendrycks2021measuring} and good recommendation~\citep{kano2019good}, admit multiple correct or satisfactory solutions. 
Such problems motivate the study of bandit learning with multiple correct answers~\citep{degenne2019pure}, on which the benefit of the abundance of optimal arms is not captured by the classical $\Theta(\sqrt{KT})$ regret. 
Prior work exploits this benefit through sub-sampling: select approximately $K/A$ arms and run a standard bandit algorithm on the selected subset~\citep{zhu2020regret,de2021bandits}, yielding an expected regret of $\tilde{O}(\sqrt{KT/A})$, where $A$ is the number of optimal arms. 
However, this rate does not fully capture the benefit when $A$ is close to $K$. Indeed, for $A=K-1$, if we randomly sample two arms and apply any minimax-optimal bandit algorithm, with probability only $O(1/K)$ one of the two arms is sub-optimal, leading to a regret of $O(\sqrt{T}/K)$. This discrepancy raises the question: 
\begin{center}
    \emph{What is the minimax regret of multi-armed bandit with multiple optimal arms?}
\end{center}
In this paper, we solve this question up to logarithmic factors. We provide a refined analysis of previous sub-sampling algorithms, and a novel near-matching lower bound for the regime $A>K/4$ that was not covered by prior work. Formally, writing the class of $K$-armed bandit with $A$ optimal arms as $\class{K}{A}$, our upper bound and lower bound together gives
\begin{align*}
    \inf_{\alg}\sup_{ \nu \in \class{K}{A}} \regret_{\alg, \nu}(T) = \tilde{\Theta}\bigg(\frac{K-A}{\sqrt{KA}}\sqrt{T}\bigg),
\end{align*}
where $\regret_{\alg, \nu}(T)$ denotes expected regret under algorithm $\alg$ and bandit instance $\nu$. This characterization recovers the previous rate when $K-A = \Omega(K)$ and sharpens the previous regret upper bound by a factor of $\tilde{\Theta}((K-A)/K)$ when $K-A = o(K)$.

The minimax characterization also raises an adaptation question that whether a learner can attain the appropriate rate without knowing $A$~\citep{zhu2020regret,de2021bandits}. \citet{de2021bandits} showed that, for well separated $B=o(A)$, if $A+B \leq K$, then no strategy can simultaneously achieve minimax-optimal regret on both $\class{K}{A}$ and $\class{K}{B}$. In this paper, we provide a novel analysis that removes the constraint on $A+B \leq K$, showing that any strategy achieving near-optimal regret on $\class{K}{A}$ must pay an additional multiplicative factor of $\tilde{\Omega}(\sqrt{A/B})$ for regret on $\class{K}{B}$, thus providing a near-complete picture of non-adaptivity. Our contributions are summarized as follows
\begin{itemize}[leftmargin=*]
    \item We provide a refined analysis of previous algorithms for bandit with multiple optimal arms. Together with our novel lower bound for the regime $A > K/4$, we establish the minimax regret for $K$-armed bandit with $A$ optimal arms, up to logarithmic factors, for all $K$, $T \geq K$ and $A \in [1,K-1]$.
    \item We show that, for any well separated $A > B$, any strategy achieving near-optimal regret on $\class{K}{A}$ will demonstrate a multiplicative loss of $\tilde{\Omega}(\sqrt{A/B})$ compared to the minimax-optimal rate on $\class{K}{B}$, necessitating the knowledge of $A$ to achieve the near-optimal regret.
    \item Technically, standard two-point lower-bound arguments based on disjoint sets of optimal arms become ineffective when the number of optimal arms is close to $K$. We overcome this obstacle by considering a larger collection of symmetric instances and comparing them with a common all-equal reference instance.
    Together with a multiplicative change-of-measure argument, this approach enables us to establish worst-case hardness results even when the instances necessarily share optimal arms.
    
\end{itemize}
\noindent\textbf{Notation.} The set $\cA$ are assumed to be finite throughout the paper.
For nonnegative sequences $\{x_n\}$ and $\{y_n\}$, we write $x_n = O(y_n)$ if $\limsup_{n\to\infty}{x_n}/{y_n} < \infty$, $x_n = o(y_n)$ if $\limsup_{n\to\infty}{x_n}/{y_n} =0$, $y_n = \Omega(x_n)$ (interchangeably written as $y_n \gtrsim x_n$) if $x_n = O(y_n)$, and $y_n = \Theta(x_n)$ if $x_n = O(y_n)$ and $x_n = \Omega(y_n)$. We further employ $\tilde{O}(\cdot), \tilde{\Omega}(\cdot)$, and $\tilde{\Theta}(\cdot)$ to hide $\polylog$ factors. 
For a pair of probability measures $P \ll Q$ on the same space, we use $\mathrm{TV}(\PP, \QQ)$ to denote their total variation, and $\kl{P}{Q} \coloneqq \int \log({\ud P}/{\ud Q})\ud P$ to denote their KL divergence. 
We denote $[N] \coloneqq \{1, \cdots, N\}$ for any positive integer $N$. Boldfaced lowercase letters, like $\xb$ are reserved for vectors, whose $i$-th entry is represented as $x_i$.

\section{Related Work}

\paragraph{Bandits with Multiple Optimal Arms.}  MABs with multiple optimal arms has been widely investigated in both regret minimization and pure exploration settings. 
In the pure exploration setting, the existence of multiple correct answer largely reduces the sample complexity~\citep{de2021bandits,katz2020true}, and the asymptotical optimal instance-dependent rate for fixed confidence best arm identification was established when number of optimal arm $A$ is unknown~\citep{degenne2019pure} and known~\citep{truong2026optimal}. 
In the regret minimization setting, when restricting algorithms to be uniformly fast convergent ~\citep[Definition 1]{garivier2019explore}, the instance dependent regret grows with $\sum_{a\in \cA \setminus \cA^*} \Delta_a^{-1} \log T$ asymptotically as $T \to \infty$, linearly with the number of suboptimal arms $K-A$~\citep{lai1985asymptotically,burnetas1996optimal}. The uniform convergence requirement is relaxed in a subsequent line of work that considered bandits with large or even infinite number of arms~\citep{chaudhuri2017pac,chaudhuri2018quantile,aziz2018pure,de2021bandits,zhu2020regret}. Specializing back to our setting, their results give an instance-dependent regret of $K\log T \log(\Delta^{-1}) (A\Delta)^{-1}$~\citep{de2021bandits} and correspondingly, $\tilde{O}(\sqrt{KT/A})$ worst case regret, which is near-optimal when $K-A = \Omega(K)$~\citep{zhu2020regret}, via an algorithmic sub-sampling~\citep{wang2008algorithms,bayati2020unreasonable,kalvit2020finite}. Despite these advances, the minimax characterization of regret is restricted to the regime $K - A = \Omega(K)$, since the standard lower bound techniques like Le Cam's method are ineffective at $K-A = o(K)$. In offline setting, \citet{ji2026optimal} provided a near-optimal characterization for all $A \in (0,K)$, when restricting to uniform behavior policy, through a multiplicative change-of-measure appeared in~\citet{degenne2019pure,garivier2021nonasymptotic}. However, the corresponding result in online setting is not yet clear.

\paragraph{Adaptivity in Bandits.}
Adaptivity asks whether a learner can match the regret attainable when a problem parameter is known, without receiving that parameter as input.
Some examples of affirmative results of adaptivity include adapting to unknown misspecification in contextual bandit~\citep{foster2020adapting}, unknown weak corruption budgets~\citep{wei2022model,liu2024corruption}, and unknown noise variance~\citep{zhao2023variance}. Despite these results, other adaptation problems demonstrate intrinsic hardness, leading to a number of hardness results. Examples in this category include adapting to unknown corruption tolerance in strong corruption~\citep{bogunovic2021stochastic,he2022nearly}, unknown smoothness classes in continuum-armed bandit~\citep{locatelli2018adaptivity,hadiji2019polynomial} and nonparametric contextual bandit~\citep{gur2022smoothness}, unknown reward range~\citep{hadiji2023adaptation}, unknown reward noise distribution~\citep{ashutosh2021bandit,genalti2024adaptive}, and unknown smallest realizable function class within a nested function class~\citep{zhu2022pareto,marinov2021pareto}.
For the proportion of optimal arms that is studied in this paper, \citet{zhu2020regret} prove non-adaptivity and characterize Pareto-optimal gap-independent rates, and \citet{de2021bandits} establish a gap-dependent non-adaptivity lower bound in the reservoir model.
Compared to their proof, our hardness result also cover regimes with $A+B>K$, where $A$ and $B$ are the respective numbers of optimal arms in two $K$-armed instance classes.
Any two instances, one from each class, then share an optimal arm, so the earlier constructions in~\citet{zhu2020regret,de2021bandits} based on disjoint optima cannot establish non-adaptivity in this regime.

\section{Problem Setup}\label{sec:setup}

We consider $K$-armed multi-armed bandit (MAB) with $A$ optimal arms and denote it  with  a tuple $(\cA, r, T)$, where $ K := |\cA| < \infty$ is the number of actions, $r: \cA \to [0, 1]$ is the reward function unknown to the learner,  and $T \geq 1$ is the total number of interactions. At each round $t \in [T]$, the learner selects an action $a_t \in \cA$ according to some strategy, and observes a noisy reward $r_t = r(a_t) + \varepsilon_t$, where $\varepsilon_t$ is $1$-sub-Gaussian~\citep[Definition~5.2]{lattimore2020bandit}. The learner's goal is to minimize the regret:
\begin{align*}
    \regret(T) = \sum_{t=1}^T r^* - \sum_{t=1}^T \EE [r(a_t)],
\end{align*}
where $r^* = \max_{a \in \cA} r(a)$ and the expectation is taken with respect to the algorithm and the noises. In this paper, we consider the case where the set of optimal arms $\aopt := \{a \in \cA: r(a) = \max_{a' \in \cA} r(a')\}$ has a known cardinality $A \in [1, K-1]$. We denote $\asubopt = \cA \setminus \aopt$. Our goal is to characterize the minimax optimality of the instance class
\begin{align}
    \inf_{\alg}\sup_{ \nu \in \class{K}{A}} \regret_{\alg, \nu}(T), \label{eq:minimax-regret}
\end{align}
for given $K,A,T$, where $\class{K}{A}:=\{\nu: | \aopt(\nu) | = A\}$.

\section{Algorithms and Upper Bounds}\label{sec:upper-bound}

In this section, we provide the algorithm for learning a $K$-armed bandit with $A$ optimal arms. The core algorithm design, which first samples a portion of arms and then run a bandit algorithm on the sampled set, follows the previous algorithm design for bandit with many optimal arms~\citep{zhu2020regret,de2021bandits}. In this paper, we provide a sharper analysis, achieving a near-optimal rate on the entire range $A \in [1, K-1]$.
We first summarize the algorithm in Algorithm~\ref{algorithm:few-optimal-arms}. At a high level, the algorithm first randomly samples a subset $\cB \subseteq \cA$ with cardinality 
\begin{align}
    |\cB| = L = \begin{cases}  \min \bigg\{ K-A+1, \bigg\lceil\frac{\log(2T)}{\log(K(K-A)^{-1})} \bigg\rceil \bigg\} & \text{if }A \le K / 2 
    \\  \min \{K-A+1, \lceil \log_2 T + 1 \rceil \} & \text{otherwise}
    \end{cases}. \label{eq:subsample-size} 
\end{align}
Then Algorithm~\ref{algorithm:few-optimal-arms} runs a minimax-optimal algorithm on $\cB$. The following theorem characterizes the regret upper bound of Algorithm~\ref{algorithm:few-optimal-arms}.

\begin{theorem}\label{thm:upper-bound}
The regret of Algorithm~\ref{algorithm:few-optimal-arms} with $L$ set in \eqref{eq:subsample-size} on any $K$-armed bandit instance with horizon $T \geq K$ and exactly $A \in [1,K-1]$ optimal arms admits an upper bound
\begin{align*}
    \regret(T) = \tilde{O}\bigg(\frac{K-A}{\sqrt{KA}}\sqrt{T}\bigg).
\end{align*}
\end{theorem}
Although we analyze the same algorithmic framework as \citet{zhu2020regret}, we establish a strictly superior regret bound via a sharper theoretical analysis. Specifically, our result shaves off a factor of $(K-A)/K$ from the previous $\widetilde{O}(\sqrt{KT/A})$ bound. This improvement is particularly pronounced in bandit instances with a large proportion of optimal arms.
\begin{remark}
\label{rmk:high-prob}
Theorem~\ref{thm:upper-bound} provides an improved guarantee solely in expectation. One may naturally ask whether this improvement can also be attained with high probability. In Appendix~\ref{app:high-prob}, we show that such an extension is impossible. We prove that any algorithm satisfying the regret bound of Theorem~\ref{thm:upper-bound} must suffer $\widetilde{\Omega}(\sqrt{T})$ regret with non-vanishing probability. 
This finding highlights an inherent separation between expected and high-probability regret in the setting of bandits with multiple optima. 
\end{remark}

\subsection{Proof Overview of Theorem~\ref{thm:upper-bound}}

We outline the proof of Theorem~\ref{thm:upper-bound}, deferring the full details to Appendix~\ref{app:proof-upper-bound}. We distinguish three cases according to the composition of the randomly selected subset $\cB$:
\begin{enumerate}[leftmargin=*]
    \item The subset $\cB$ contains no optimal arms, corresponding to the event $\eventmiss := \{\cB \cap \aopt = \varnothing\}$. In this case, we use the trivial regret bound of $T$.
    \item The subset $\cB$ contains both optimal and suboptimal arms, denoted by event $\eventsubopt := \{\cB \cap \asubopt \neq \varnothing\} \cap \{\cB \cap \aopt \neq \varnothing\}$. In this case, running MOSS for $T$ rounds incurs expected regret $O(\sqrt{LT})$.
    \item The subset $\cB$ contains only optimal arms, corresponding to the event $\eventopt := \{\cB \subseteq \aopt\}$. In this case, the regret is zero.
\end{enumerate}

\paragraph{Previous Analysis. }Previously, \citet{de2021bandits,zhu2020regret} distinguish between missing all optimal arms and including at least one. Choosing $L$ of order $K/A$, up to suitable logarithmic factors, makes $\PP[\eventmiss]$ negligible. On the complementary event $\eventsubopt \cup \eventopt$, the expected regret is bounded by $O(\sqrt{LT})$, yielding an overall bound of $\tilde{O}(\sqrt{KT/A})$. This rate is optimal up to logarithmic factors when $K-A=\Omega(K)$.

\paragraph{Our Approach. }The previous analysis is not sharp when $A$ is close to $K$, because it does not exploit the zero regret incurred on $\eventopt$. In this regime, each sampled arm is suboptimal with probability only $(K-A)/K$, and the subset size is $L=\tilde{O}(1)$. A union bound therefore gives
\begin{align*}
\PP[\eventsubopt]
\leq \PP[\eventmiss \cup \eventsubopt]
\leq L\frac{K-A}{K}
= \tilde{O}\bigg(\frac{K-A}{K}\bigg).
\end{align*}
Combining this probability bound with the conditional expected regret bound $\tilde O(\sqrt{T})$ on $\eventsubopt$ yields
\begin{align*}
\regret(T) \lesssim \PP[\eventmiss]T +\PP[\eventsubopt]\sqrt{LT} = \tilde{O}\bigg(\frac{K-A}{K}\sqrt{T}\bigg),
\end{align*}
where the contribution from $\eventmiss$ is made negligible by the choice of $L$. This establishes the desired rate when $A$ is close to $K$.







\begin{algorithm}[t]
\caption{$\algname$~\citep{zhu2020regret,de2021bandits}}\label{algorithm:few-optimal-arms}
    \begin{algorithmic}[1]
    \REQUIRE Action set $\cA$, number of total rounds $T$, number of optimal arm $A$ if $A \le K / 2$.
    \STATE Set $L$ as in~\eqref{eq:subsample-size} and sample a subset $\cB \subseteq \cA$ of $L$ distinct arms uniformly without replacement.
    \STATE Run any minimax-optimal $L$-armed bandit algorithm, such as MOSS~\citep{audibert2009minimax}, on $\cB$ for $T$ rounds.
    \end{algorithmic}
\end{algorithm}


\section{Lower Bounds}\label{sec:lower-bound}

In this section, we provide a lower bound of~\eqref{eq:minimax-regret}. We begin with the regime in which the number of optimal arms is small relative to $K$. The lower bound in this regime is suggested by previous work~\citep{zhu2020regret} and follows from a standard Le Cam two-point argument (see, e.g., \citealt[Theorem~15.2]{lattimore2020bandit}). 

\begin{theorem}\label{thm:lower-few-optimal-arms}
Given any $K$, $1 \leq A \leq K/4$ and $T \geq (K-A)/A$, for any algorithm, there exists a $K$-armed bandit with $A$ optimal arms, on which the algorithm suffers from $\tilde{\Omega}(\sqrt{A^{-1}KT})$ regret.
\end{theorem}

We next consider the regime in which $A$ is comparable to $K$, where the proof techniques of previous lower bounds~\citep{de2021bandits,zhu2020regret} is not effective. The main difficulty arises when $A$ is close to $K$: the overlap between the sets of optimal arms makes a standard two-point argument ineffective. The lower bound for this regime is stated in the following theorem. 

\begin{theorem}\label{thm:lower-many-optimal-arms}
Given any $K$, $A > K/4$ and $T \geq 4/\log K$, for any algorithm, there exists a $K$-armed bandit with $A$ optimal arms, on which the algorithm suffers from $\tilde{\Omega}((K-A)K^{-1}\sqrt{T})$ regret.
\end{theorem}

Together, Theorems~\ref{thm:lower-few-optimal-arms} and~\ref{thm:lower-many-optimal-arms}, combined with Theorem~\ref{thm:upper-bound}, yield
\begin{align*}
    \inf_{\alg}\sup_{ \nu \in \class{K}{A}} \regret_{\alg, \nu}(T) = \tilde{\Theta}\bigg(\frac{K-A}{\sqrt{KA}}\sqrt{T}\bigg),
\end{align*}
for all $K$, $1 \leq A \leq K-1$ and $T \geq K$. This characterizes the minimax regret of $K$-armed bandit with $A$ optimal arms up to logarithmic factors across the full range of $A$.

\subsection{Proof Overview of Theorem~\ref{thm:lower-few-optimal-arms}}

In this section, we provide a quick overview of the proof of Theorem~\ref{thm:lower-few-optimal-arms} and defer the full proof to Appendix~\ref{app:proof-lower-few-optimal-arms}. 
Starting from here, unless stated otherwise, we design the noise of rewards to be standard Gaussian.

Fix a constant $\Delta > 0$ to be specified later. We apply the Le Cam two-point method with two problem instances. For the first instance, we consider $\cS_0 \subset \cA$ such that $|\cS_0| = A$ and define the instance $\nu_0$ whose reward given by $r_0(a) = 1/2 + \Delta$ for $a \in \cS_0$ and $r_0(a) = 1/2 $ otherwise. For any algorithm $\alg$, let $N_{\cS}(T)$ be the number of pulls within $\cS$ up to time $T$. By the pigeonhole principle, there exists some $\cS_1 \subset \cA \setminus \cS_0$ such that $|\cS_1| = A$ and
\begin{align*}
    \EE_{\nu_0, \textsf{Alg}}[N_{\cS_1}(T)] \leq \frac{TA}{K-A} \lesssim \frac{TA}{K},
\end{align*}
where the expectation is taken over the distribution jointly given by instance $r_0$ and \textsf{Alg}. 
Now we construct the second instance $\nu_1$ as $r_1(a) = 1/2 + 2\Delta$ for all $a \in \cS_1$ and $r_1(a) = r_0(a)$ otherwise. 
Under this construction, the distribution induced by the two instances is close, as the KL-divergence between the trajectory distributions $\PP_0$ and $\PP_1$ is given by
\begin{align}
     \sum_{a \in \cA} \EE_0[N_a(T)]\KL\big(\cN(r_0(a), 1) \| \cN(r_1(a), 1) \big) \lesssim \frac{TA}{K} \Delta^2 , \label{eq:overview-slow-kl}
\end{align}
where we adopt the shorthand $\PP_i:=\PP_{\nu_i, \textsf{Alg}}$ to denote the probability distributions over trajectories induced by the interaction between algorithm \textsf{Alg} and instances $\nu_i$ for $i \in \{0, 1\}$. Now we pick $\Delta \sim \sqrt{K/(AT)}$, yielding $\kl{\PP_0}{\PP_1} = O(1)$, indicating that under algorithm \textsf{Alg}, it is hard to distinguish $\nu_1$ from $\nu_0$\footnote{In general, the KL-divergence $\kl{\PP}{\QQ}$ between two distributions $\PP$ and $\QQ$ being a constant indicates that $\PP$ and $\QQ$ cannot be reliably distinguished.}.

However, we see that the optimal actions of $\nu_0$ and $\nu_1$ are disjoint. In other words, at each step, the learner incurs a regret of $\Delta$ on either $\nu_0$ or $\nu_1$ if the learner cannot distinguish $\nu_0$ and $\nu_1$. Since under our construction, the learner cannot reliably distinguish $\nu_0$ and $\nu_1$ until time step $T$, it must suffer from a regret of $\Omega(T\Delta) = \Omega(\sqrt{KT/A})$ on one of the instances.

\subsection{Proof Overview of Theorem~\ref{thm:lower-many-optimal-arms}}\label{sec:overview-lower}

The preceding proof of Theorem~\ref{thm:lower-few-optimal-arms} relies on two instances whose trajectory distributions are difficult to distinguish but whose \emph{sets of optimal arms are disjoint}. Disjointness is essential: if the instances share an optimal arm, an algorithm can repeatedly pull that arm and incur zero regret on both instances without distinguishing between them.

When $A>K/2$, however, any two instances in $\class{K}{A}$ share at least one optimal arm. An argument based solely on such a pair therefore cannot establish the desired lower bound. To overcome this obstacle, we consider a larger collection of instances, where each instance is given by subtracting the reward by a small gap $\Delta$ on some subset of $\cA$ with cardinality $K-A$. Formally, we state the instances as indexed by
\begin{align*}
\mu \in \cV_A:= \bigg\{ \bmu \in \{1 / 2, 1 / 2 - \Delta\}^K \bigg| \sum_{i=1}^K \ind(\mu_i= 1 / 2) = A \bigg\},
\end{align*}
where $\Delta = \tilde{\Theta}(1/\sqrt{T})$. Given any $\mu, \nu \in \cV_A$, the reward distribution on any of the arms differs by only $\Delta$, making any algorithm unable to distinguish between any individual instance among $\cV_A$. To formally state this, we introduce an auxiliary instance $\theta$ whose reward vector is given by $\btheta = (1/2, \cdots, 1/2)^\top \in \RR^K$, and let $\cV = \cV_A \cup \{\theta\}$. 

We first show that, for any $\mu \in \cV_A$, the induced distribution is sufficiently close to $\theta$. We do this by computing the KL-divergence between the two distribution $\PP_{\mu}$ and $\PP_{\theta}$, induced by a fixed algorithm $\alg$, and $\mu$, $\theta$ correspondingly. By the chain rule of KL-divergence, we see that
\begin{align*}
    \kl{\PP_\theta}{\PP_\mu} = \frac{\Delta^2}{2}\EE_{a \in \asubopt(\mu)}[N_{a}(T)]  \leq T\Delta^2 = O(1),
\end{align*}
where $N_a(T)$ denotes the number of pulls on arm $a$ up to $T$ interactions, and the inequality holds due to $\EE_{a \in \asubopt(\mu)}[N_{a}(T)] \leq T$. Recall that, an $O(1)$ KL-divergences indicates that $\PP_\mu$ and $\PP_\theta$ cannot be reliably distinguished. This implies that, $\alg$ cannot behaves too distinctively on $\mu$ and $\theta$. To characterize this, we need the following multiplicative change-of-measure proposition. See Remark~\ref{rmk:multiplicative-change-of-measure} for its previous applications in proving bandit lower bounds.

\begin{proposition}\label{prop:change-of-measure-informal}
Let $\PP$ and $\QQ$ be probability measures satisfying $\PP\ll\QQ$ and define the log-likelihood ratio $L = \log(\ud\PP/\,\ud\QQ) $. Then, for every bounded random variable $X \in [0,b]$ and every $\gamma\in\RR$,
\begin{align}
    \EE_{\PP}[X] \leq e^\gamma\EE_{\QQ}[X] +  b\PP[L>\gamma].
    \label{eq:change-of-measure-informal}
\end{align}
\end{proposition}

We apply Proposition~\ref{prop:change-of-measure-informal} with $\QQ = \PP_{\mu}$, $\PP = \PP_{\theta}$, and $\gamma \approx \kl{\PP_{\theta}}{ \PP_{\mu}}$. Since $\EE_{\theta}[L] = \kl{\PP_{\theta}}{ \PP_{\mu}}$, a careful selection of $\gamma$ actually allows $\PP[L>\gamma]$ negligible. The application of Proposition~\ref{prop:change-of-measure-informal} results in the following change-of-measure inequality
\begin{align}
    \EE_{\mu}[X] \gtrsim \EE_{\theta}[X] - o\bigg(\frac{T}{K}\bigg),\label{eq:change-of-measure-applied}
\end{align}
for any random variable $X \in [0,T]$.

However, since on the instance $\theta$, the rewards on all the arms are equal, the algorithm $\alg$ must allocate at least $T(K-A)/K$ pulls on one of the subset $\cS \subset \cA$ with $|\cS| = K-A$ by symmetry. Now we consider the instance $\mu \in \cV_A$ such that $\asubopt(\mu) = \cS$. We transfer the behavior of $\alg$ on $\theta$ to $\mu$ via~\eqref{eq:change-of-measure-applied}. In particular, substitute $X$ with $N_{\asubopt(\mu)}(T) = \sum_{t=1}^T \ind\{a_t \in \asubopt(\mu)\}$, we obtain
\begin{align*}
    \EE_{\mu}[N_{\asubopt(\mu)}(T)] \gtrsim \EE_{\theta}[N_{\asubopt(\mu)}(T)] - o\bigg(\frac{T}{K}\bigg) \geq \frac{K-A}{K} T - o\bigg(\frac{T}{K}\bigg) = \Omega\bigg(\frac{K-A}{K}T \bigg).
\end{align*}
Finally, recall that $\regret_{\mu, \alg}(T) = \Delta\EE_{\mu}[N_{\asubopt(\mu)}(T)]$, we see that
\begin{align*}
    \regret_{\mu, \alg}(T) \gtrsim \Omega\bigg(\frac{K-A}{K}T\Delta \bigg) = \tilde\Omega\bigg(\frac{K-A}{K}\sqrt{T} \bigg),
\end{align*}
where the last equation is obtained from $\Delta = \Theta(1/\sqrt{T})$. This gives the desired lower bound.

\section{Impossibility to Adapt to Number of Optimal Arms}\label{sec:non-adaptivity}

When $A = o(K)$, Algorithm~\ref{algorithm:few-optimal-arms} requires $A$ to know the sample size $L$ to achieve the near-minimax rate. In this section, we show that such knowledge of the $A$ is necessary. In particular, we show that, any algorithm cannot achieve (near)-minimax rate simultaneously on $\class{K}{A}$ and $\class{K}{B}$, if at least one of $A$ or $B$ is $o(K)$. From a technical perspective, we split the results into two disjoint regimes: $A + B  \leq K$ and $A+B > K$. We first consider the simpler case $B \leq K-A$.

\begin{theorem}\label{thm:non-adaptivity-A-small}
Let $1 \leq A \leq K-1$, and consider any algorithm $\mathsf{Alg}$ satisfying that there exists some $\Gamma_{K,T} = \polylog(K,T)$ such that for all $T > 0$, 
\begin{align*}
    \sup_{\nu:\, |\aopt(\nu)|=A} \regret_{\nu,\mathsf{Alg}}(T) \leq \Gamma_{K,T}\frac{K-A}{\sqrt{KA}}\sqrt{T}.
\end{align*}
Then $\forall$ $B \leq \min \{A/(8\Gamma_{K,T}), K-A\} $ and $T \geq \max \{256B^2\Gamma^2_{K,T}/(KA), KA/(16B^2\Gamma_{K,T}^2)\}$, there exists a $K$-armed bandit instance $\nu$ with exactly $B$ optimal arms such that
\begin{align}
    \regret_{\nu,\mathsf{Alg}}(T) \geq \Omega \Bigg(\frac{\sqrt{KA}}{B\Gamma_{K,T}} \sqrt{T} \Bigg). \label{eq:non-adaptivity-A-small-rate} 
\end{align}
\end{theorem}

Theorem~\ref{thm:non-adaptivity-A-small}, in terms of both technique and results statement, follows \citet[Theorem 3]{de2021bandits}. Thus, it also admits constraints to the range of $B$ similar to~\citet{de2021bandits}. The constraint $B = \tilde{O}(A)$ is mild, as if $B = \tilde{\Omega}(A)$, then running $\algname$ with $B$ also gives a near-optimal rate on $\class{K}{A}$ up to a factor of $O(A/B) = \tilde{O}(1)$. The second constraint, $B \leq K-A$, however, is intrinsic to the proof technique. Relaxing this constraint requires fundamentally different approach. We provide an overview of the proof technique of Theorem~\ref{thm:non-adaptivity-A-small}, along with its limitation, in Section~\ref{sec:non-adaptivity-A-small-overview}.

We next state the results in the regime $B \geq K-A$ as follows,

\begin{theorem}\label{thm:non-adaptivity-A-large}
Let $K/2 \leq A \leq K-1$, and consider any algorithm $\mathsf{Alg}$ achieving near-optimal regret on $\class{K}{A}$, that is, for all $T>0$ and some fixed $\Gamma_{K,T}=\polylog(K,T)$
\begin{align*}
    \sup_{\nu:\,|\aopt(\nu)|=A} \regret_{\nu,\mathsf{Alg}}(T) \leq \Gamma_{K,T}  \frac{K-A}{K} \sqrt{T}.
\end{align*}
Then, for any $B$ satisfying $K-A \leq B \leq K/2$ and $B \sqrt{\log(32B/(K-A))} \leq K/(256 \Gamma_{K,T})$ and any $T \geq \max \{K^2 \big(B^2 \Gamma_{K,T}^2 \log(32B/(K-A))\big)^{-1}, 4\}$, there exists a $K$-armed bandit instance $\nu$ with exactly $B$ optimal arms such that
\begin{align}
    \regret_{\nu,\mathsf{Alg}}(T) \geq \tilde{\Omega} \Bigg( \frac{\sqrt{KA}}{ B\Gamma_{K,T}}\sqrt{T} \Bigg). 
    \notag
\end{align}
\end{theorem}

Ignoring all logarithmic factor, Theorem~\ref{thm:non-adaptivity-A-small} and Theorem~\ref{thm:non-adaptivity-A-large} together show that, when $B = O(A)$, any algorithm that is nearly minimax-optimal on $\class{K}{A}$ must suffer a regret of at least $\tilde{\Omega}(\sqrt{KAT}/B)$, which demonstrates a multiplicative loss of $\tilde{\Omega}(\sqrt{A/B})$ compared to the minimax-optimal rate $\tilde{O}(\sqrt{KT/B})$ on $\class{K}{B}$. These results indicate that, algorithm agnostic to the number of optimal arms cannot simultaneously attain the near-optimal regret on $\class{K}{A}$ and $\class{K}{B}$ if $B = o(A)$, necessitating the knowledge of the number of optimal arms up to logarithmic factors.

\begin{remark}
The tradeoff between $\class{K}{A}$ and $\class{K}{B}$ when $B=o(A)$ motivates the notion of a Pareto-optimal algorithm~\citep{zhu2020regret}: one for which no other algorithm achieves better performance pointwise over all $1\leq A\leq K-1$. Designing a Pareto-optimal algorithm for our setting remains an interesting direction for future work.
\end{remark}




\subsection{Proof Overview of Theorem~\ref{thm:non-adaptivity-A-small}}\label{sec:non-adaptivity-A-small-overview}


In this section, we provide an overview of the proof of Theorem~\ref{thm:non-adaptivity-A-small}. At a high level, the proof follows~\citet{de2021bandits} and considers the two-point-type instances, and then shows that algorithm cannot reliably distinguish these two instances through a TV-based change-of-measure argument.

For simplicity, we constrain us to the case $K-A = \Omega(K)$. At a high level, our goal is to consider two instances $\nu_0 \in \class{K}{A}$ and $\nu_1 \in \class{K}{B}$, such that they are similar enough, yet any algorithm must suffer from sufficiently large regret on at least one of them. To do this, we consider partition $\cA$ as follows, where $\cS_0$ is fixed and $\cS_1$ determined later
\begin{align*}
    \cA = \cS_0 \cup \cS_1 \cup \cS_2, \quad \text{where  } |\cS_0| = A, \, |\cS_1| = B, \text{ and }, |\cS_2| = K-A-B.
\end{align*}
We use $r_0$ for the reward mean of $\nu_0$ and $r_1$ for the reward of $\nu_1$, which are given by
\begin{align*}
    r_0(a) = \begin{cases}
        1/2, & a \in \cS_0 \\
        1/2 - \Delta, & a \in \cS_1 \\
        1/2 - \Delta, & a \in \cS_2
    \end{cases}, \qquad 
    r_1(a) = \begin{cases}
        1/2, & a \in \cS_0 \\
        1/2 + \Delta, & a \in \cS_1 \\
        1/2 - \Delta, & a \in \cS_2
    \end{cases}.
\end{align*}
It is easy to verify that $\nu_0 \in \class{K}{A}$ and $\nu_1 \in \class{K}{B}$. The key observation is that, since arms in $\cS_0^\complement$ is sub-optimal under $\nu_0$, any algorithm $\alg$ optimal on $\class{K}{A}$ cannot allocate too many pulls on $\cS_0^\complement$, the same for one of its subset $\cS_1$. In particular, fixing the algorithm and dropping the subscripts of $\alg$, if $\regret(T) \leq R_A$, we must observe
\begin{align*}
    \EE_0[N_{\cS_0^\complement}(T)] \leq \frac{R_A}{\Delta},
\end{align*}
where $\EE_0$ is a short hand for $\EE_{\nu_0}$ and similar abbreviation applies to $\nu_1$. By symmetry, there exist a subset $\cS \subseteq \cS_0^\complement$ such that $|\cS| = B$ and 
\begin{align*}
     \EE_0[N_{\cS}(T)] \leq \frac{BR_A}{(K-A)\Delta}.
\end{align*}
Without loss of generality, we pick $\cS_1 = \cS$. With this construction, we see that (i) since $\alg$ does not allocate sufficient pulls on $\cS_1$, $\alg$ cannot reliably distinguish $\nu_1$ from $\nu_0$; and (ii) the optimal actions of $\nu_0$ and $\nu_1$ are disjoint, therefore $\alg$ cannot be optimal simultaneously on them.

Formally, we first show that $\alg$ cannot reliably distinguish $\nu_1$ from $\nu_0$, through a computation of the KL-divergence between there distribution. Taking $\Delta = \Theta((K-A)/(BR_A))$
\begin{align*}
    \kl{\PP_0}{\PP_1} = \EE_0[N_{\cS_1}(T)] \cdot (2\Delta) ^2 = O\bigg(\frac{BR_A \Delta}{K-A} \bigg) = O(1).
\end{align*}
In this case, the distribution $\PP_0$ and $\PP_1$ is close enough, therefore the random variable $N_{\cS_1}(T)$ behaves similarly under these two distribution. In particular, a change-of-measure leads to 
\begin{align*}
    \EE_1[N_{\cS_1}(T)] \leq     \EE_0[N_{\cS_1}(T)] + T \mathrm{TV}(\PP_0,\PP_1) \leq \EE_0[N_{\cS_1}(T)] + T \sqrt{\frac{1}{2}\kl{\PP_0}{\PP_1}},
\end{align*}
where the last inequality holds due to Pinsker's inequality. Plugging in $\Delta = \Theta((K-A)/(BR_A))$ into $\EE_0[N_{\cS_1}(T)] \leq BR_A/((K-A)\Delta)$ and $\kl{\PP_0}{\PP_1} = O(1)$ gives
\begin{align*}
    \EE_1[N_{\cS_1}(T)] \leq \frac{B^2 R_A^2}{(K-A)^2} + O(T) = \tilde{O}\bigg(\frac{B^2}{(K-A)^2}\frac{(K-A)^2 T}{KA} + T \bigg) = \tilde{O}(T),
\end{align*}
where the first equation holds by plugging in the rate in Theorem~\ref{thm:upper-bound}. Finally, a careful selection of constants enables $T - \EE_1[N_{\cS_1}(T)] = \Omega(T) $, thus
\begin{align*}
    \regret_{\nu_1}(T) \geq (T-\EE_1[N_{\cS_1}(T)]) \Delta = \Omega(T \Delta) = \Omega\bigg(\frac{T(K-A)}{BR_A} \bigg) = \tilde{\Omega}\bigg(\frac{\sqrt{KA}}{B}\sqrt{T}\ \bigg),
\end{align*}
which gives the desired rate.




\subsection{Proof Overview of Theorem~\ref{thm:non-adaptivity-A-large}}\label{sec:non-adaptivity-A-large-overview}

Similar to Theorem~\ref{thm:lower-few-optimal-arms}, the proof of Theorem~\ref{thm:non-adaptivity-A-small} relies on the two-point construction that requires \emph{the two instances have disjoint sets of optimal arms}. 
However, in the case $A + B > K$, any instance $\nu_0 \in \class{K}{A}$ and $\nu_1 \in \class{K}{B}$ must share at least one optimal arm $a \in \aopt(\nu_0) \cap \aopt(\nu_1)$, preventing an application in this regime.

We overcome this obstacle by showing that any near-optimal algorithm must exhibit a non-uniform exploration profile akin to that of Algorithm~\ref{algorithm:few-optimal-arms}, therefore ruling out algorithms that concentrate their pulls on a fixed arm or set of arms. 
Yet even under a fixed algorithm, different instances in $\class{K}{A}$ induce different trajectory distributions, complicating a unified characterization of its behavior. We therefore characterize the algorithm's exploration profile on a common reference instance $\theta$, in which all arms have the same mean reward. 
Such uniformed characterization is enabled by the change-of-measure in Proposition~\ref{prop:change-of-measure-informal}, if we select the instances in $\class{K}{A}$ close enough to $\theta$.
Formally, we have the following lemma.

\begin{lemma}\label{lem:exploration-profile-upper-bound}
Let $A=\widetilde{\Omega}(K)$, and suppose that $\alg$ is minimax-optimal on $\class{K}{A}$ up to logarithmic factors, namely
\begin{align*}
    \sup_{\nu:\,|\aopt(\nu)|=A}
    \regret_{\nu,\alg}(T)
    \leq
    \Gamma_{K,T}\frac{K-A}{\sqrt{KA}}\sqrt{T},
\end{align*}
where $\Gamma_{K,T}=\operatorname{polylog}(K,T)$. Let $\theta$ denote the all-equal instance. Then, ignoring logarithmic factors, for every $m\in[T]$, and every $\cS \subset \cA$ with $|\cS| = K-A$,
\begin{align*}
    \PP_{\theta} \big[N_{\cS}(T)\geq m\big] \lesssim \widetilde{O}\bigg(\frac{K-A}{K}\sqrt{\frac{T}{m}}\bigg).
\end{align*}
Moreover, for any $\cU \subseteq \cA$ such that $|\cU| = B \geq K-A$, we have
\begin{align}
    \PP_{\theta} \big[N_{\cU}(T)\geq m\big] \lesssim \widetilde{O}\bigg(\frac{B}{K}\sqrt{\frac{T}{m}}\bigg). \label{eq:exploration-profile-aggregated}
\end{align}
\end{lemma}

Lemma~\ref{lem:exploration-profile-upper-bound} reveals a mismatch between the exploration profiles required for near-optimality on $\class{K}{A}$ and $\class{K}{B}$ in the regime $K-A\leq B\ll K$. To illustrate this mismatch, suppose for simplicity that $B$ divides $K$, and partition the arms of $\theta$ into $K/B$ groups of size $B$, treating each group as a super-arm. We compare the exploration levels that algorithms $\alg_A$ and $\alg_B$, near-optimal on their respective instance classes, can attain with high probability on each fixed super-arm.

For $\alg_B$, we can prove that for every fixed $\cU\subseteq\cA$ with $|\cU|=B$, there is a threshold
\begin{align*}
m_B=\widetilde{\Omega}\bigg(\frac{TB}{K}\bigg), \text{ such that } N_{\cU}(T)\geq m_B,
\end{align*}
with high probability under $\theta$. Thus, at the level of super-arms, $\alg_B$ must explore relatively uniformly: each fixed super-arm receives at least $\widetilde{\Omega}(TB/K)$ pulls with high probability\footnote{See Appendix~\ref{app:exploration-profile-minimax-optimal} for the formal statement and proof.}.

In contrast, for $\alg_A$, inequality~\eqref{eq:exploration-profile-aggregated} implies that one can choose a threshold
\begin{align*}
m_A=\widetilde{O}\bigg(\frac{TB^2}{K^2}\bigg), \text{ for which, } \PP_{\theta,\alg_A}[N_{\cU}(T)\geq m_A]\leq 1/2,
\end{align*}
provided this threshold lies within the horizon. Consequently, $\alg_A$ cannot ensure more than $\widetilde{O}(TB^2/K^2)$ pulls of any fixed super-arm with high probability. When $B\ll K$, this scale is smaller than the exploration level required of $\alg_B$, up to logarithmic factors. This incompatibility yields the desired non-adaptivity result: the exploration profile imposed by near-optimality on $\class{K}{A}$ precludes the exploration needed to handle the hardest instances in $\class{K}{B}$.

\begin{remark}
Lemma~\ref{lem:exploration-profile-upper-bound} shows that, when $A$ is close to $K$, near-optimal algorithms must exhibit non-uniform exploration behavior, consistent with the sub-sampling approach in Section~\ref{sec:upper-bound}. A more precise characterization of ``sub-sampling-like'' behavior, and whether such behavior is necessary for near-optimality, remains an interesting direction for future work.
\end{remark}

Formally, we begin with any fixed $\cU \subseteq \cA$ such that $|\cU| = B$, Lemma~\ref{lem:exploration-profile-upper-bound} shows that only an exploration of level $\tilde{O}(TB^2/K^2)$ is guaranteed. Therefore, taking $m = \tilde{\Theta}(TB^2/K^2)$ with properly selected constants, we can ensure that
\begin{align*}
    \PP_{\theta} \big[N_{\cU}(T)\geq m\big] \lesssim \widetilde{O}\bigg(\frac{B}{K}\sqrt{\frac{T}{m}}\bigg) < 1 - \Omega(1).
\end{align*}
On the other side, an exploration of level $\tilde{O}(TB^2/K^2)$ is not sufficient for identifying hard-enough instance. Particularly, we consider the instance $\nu$ with reward function
\begin{align*}
    r_\nu(a) = \begin{cases}
        1/2+\Delta, & a\in\mathcal{U}\\
        1/2, & a\in\mathcal{U}^{\complement}
    \end{cases} ,
\end{align*}
where $\Delta = \tilde{\Theta}(1/\sqrt{m})$. Once again, under the event $\{N_{\cU}(T) \leq m\}$, algorithm allocates only at most $m$ to $\cU$, leading to $\kl{\PP_\theta}{\PP_\nu} = O(1)$, therefore the learner are not able to reliably distinguish $\theta$ and $\nu$ in this case. Formally, under the event $\{N_{\cU}(T) \leq m\}$, the distribution of $\PP_{\nu}$ and $\PP_{\theta}$ is close enough, so change of measure is at a cost of only multiplicative constant: 
\begin{align*}
    \PP_{\nu}\big[N_{\cU}(T)\geq m\big] \lesssim \PP_{\theta}\big[N_{\cU}(T)\geq m\big] < 1 - \Omega(1).
\end{align*}
Recall that on $\nu$, the regret is given by $\Delta\big(T -  \EE_{\nu}[N_{\cU}(T)]\big)$, therefore,
\begin{align*}
    \regret_{\nu}(T) \geq \PP_{\nu}\big[N_{\cU}(T)\leq m\big] (T-m) \Delta = \Omega(T\Delta) = \tilde{\Omega}\bigg(\frac{K}{B}\sqrt{T}\bigg),
\end{align*}
where the last inequality holds due to $\Delta = \tilde{\Theta}(1/\sqrt{m}) = \Theta(K/(B\sqrt{T}))$. This gives the desired rate and finishes the proof.

\section{Conclusion}
In this paper, we study regret minimization in $K$-armed bandits with $A$ optimal arms. We first provide a sharper analysis of the previous sub-sampling algorithm and lower bounds for the problem. These results together establishes a $\tilde{\Theta}\Big(\frac{K-A}{\sqrt{KA}}\sqrt{T} \Big)$ minimax regret. 
We also show that near-optimality imposes exploration constraints that prevent uniform adaptation across sufficiently many different models with distinct numbers of optimal arms. 
Directions for future work include closing the remaining logarithmic gaps, designing Pareto optimal algorithms and determining which aspects of the exploration behavior induced by sub-sampling are necessary for minimax optimality.

\appendix

\clearpage


\section{Proofs of Theorems in Section~\ref{sec:upper-bound}}\label{app:proof-upper-bound}

\subsection{Proof of Theorem~\ref{thm:upper-bound}}

\begin{proof}[Proof of Theorem~\ref{thm:upper-bound}]
Let $\eventmiss = \{\cB \cap \aopt = \varnothing \}$ be the event that the randomly sampled subset does not contain any optimal arms, and $\eventsubopt =  \{\cB \cap \asubopt \neq \varnothing\} \cap \{\cB \cap \aopt \neq \varnothing\}$ be the event that the randomly sampled subset contains at least one sub-optimal arm and at least one optimal arm. 
We first bound the probability of $\eventmiss$. First, if $L = K-A+1$ then every subset of $L$ arms must contain at least one optimal arm, since there are only $K-A$ suboptimal arms. Therefore, $\PP[\eventmiss]= 0$. Otherwise, if $A \leq K/2$, we directly have
\begin{align*}
L \leq \bigg\lceil \frac{\log(2T)}{\log(K/(K-A))} \bigg\rceil.
\end{align*}
In the complement case $A > K/2$, we also have
\begin{align*}
    L = \lceil \log_2 T + 1 \rceil \leq \bigg\lceil \frac{\log(2T)}{\log2} \bigg\rceil \leq   \bigg\lceil \frac{\log(2T)}{\log(K/(K-A))} \bigg\rceil.
\end{align*}
Since $\cB$ is sampled uniformly without replacement, we have
\begin{align*}
\PP[\eventmiss] = \binom{K-A}{L}\binom{K}{L}^{-1} = \prod_{i=0}^{L-1} \frac{K-A-i}{K-i} \leq  \bigg(\frac{K-A}{K}\bigg)^L,
\end{align*}
where the inequality holds due to $(K-A-i)/(K-i) \leq (K-A)/K$ for every $i\in{0,\ldots,L-1}$. Plugging in the value of $L$, we see that
\begin{align*}
\PP[\eventmiss] \leq \bigg(\frac{K-A}{K}\bigg)^L = \exp\bigg(-L\log\bigg(\frac{K}{K-A}\bigg) \bigg) \leq \exp \big( -\log(2T) \big) = \frac{1}{2T}.
\end{align*}
Combining the two cases gives $\PP[\eventmiss] \leq (2T)^{-1}$.

We then bound the probability of event $\eventsubopt \cup \eventmiss$. Trivially, we have $\PP[ \eventsubopt \cup \eventmiss ] \leq 1$. Moreover, by the fact that $\eventsubopt \cup \eventmiss = \bigcup_{a \in \asubopt} \{a \in \cB\}$
\begin{align*}
    \PP[ \eventsubopt \cup \eventmiss ] \leq \sum_{a \in \asubopt} \PP\big[\{a \in \cB\}\big] \leq \frac{L(K-A)}{K},
\end{align*}
which gives $\PP[ \eventsubopt \cup \eventmiss ] \leq 1 \wedge L(K-A)K^{-1}$

Now we are ready to bound the regret. On the event $\eventopt \coloneqq \big(\eventmiss \cup\eventsubopt\big)^\complement$, $\cB \subseteq \aopt$, therefore no regret is incurred. On the event $\eventsubopt $, the restricted bandit instance over $\cB$ contains at least one globally optimal arm. Thus, the best mean reward among the arms in $\cB$ equals $r^*$. By the minimax regret guarantee of MOSS, there exists a universal constant $C>0$ such that
\begin{align*}
\regret_{\cB}(T) \leq C\sqrt{LT}.
\end{align*}
for every realization of $\cB$ satisfying $\eventsubopt $ and $T \geq K$ (see e.g., \citealp[Theorem 9.1]{lattimore2020bandit} for a proof). Finally, On the event $\eventmiss$, the per-round regret is at most $1$, since all mean rewards belong to $[0,1]$. Therefore, the total regret is at most $T$. Combining the analysis gives that
\begin{align*}
& \regret(T) \\
& \quad = \EE\Big[ \regret(T)\ind\big( \eventopt \big) \Big] + \EE\big[ \regret(T)\ind(\eventsubopt) \big] + \EE\big[ \regret(T)\ind(\eventmiss) \big] \\
& \quad \leq C\sqrt{LT} \PP[ \eventsubopt \cup \eventmiss  ] + T\PP[\eventmiss] \\
& \quad \leq  C\sqrt{LT}\bigg(\frac{L(K-A)}{K} \wedge 1\bigg)  + \frac{1}{2}.
\end{align*}

It remains to bound $L$. Using the elementary inequality $-\log(1-x) \geq x$ for all $x\in[0,1)$, we obtain $\log(K/(K-A)) = -\log(1-A/K) \geq A/K$, therefore
\begin{align*}
L \leq \bigg\lceil \frac{\log(2T)} {\log\big(K/(K-A)\big)} \bigg\rceil \leq 1+ \frac{\log(2T)}{\log\big(K/(K-A)\big)} \leq 1+\frac{K}{A}\log(2T) = \tilde{O}\bigg(\frac{K}{A}\bigg),
\end{align*}
where the first inequality holds due to $\lceil x\rceil\leq x+1$. Substituting this inequality into the preceding regret bound yields
\begin{align*}
\regret(T) = \tilde{O} \Bigg( \bigg(\frac{(K-A)}{A} \wedge 1\bigg)\sqrt{\frac{KT}{A}} + \frac{1}{2} \Bigg) = \tilde{O}\bigg( \frac{K-A}{\sqrt{KA}}\sqrt{T}\bigg).
\end{align*}
Since this bound is achieved by a valid algorithm uniformly over all instances with exactly $A$ optimal arms, the claimed minimax upper bound follows.
\end{proof}

\section{Proofs of Theorems in Section~\ref{sec:lower-bound}}\label{app:proof-lower-bound}

\subsection{Proof of Theorem~\ref{thm:lower-few-optimal-arms}}\label{app:proof-lower-few-optimal-arms}

\begin{proof}[Proof of Theorem~\ref{thm:lower-few-optimal-arms}]
The hard instance construction largely follows Theorem 15.2 of~\citet{lattimore2020bandit} and a similar idea was also presented in~\citet{de2021bandits}. For simplicity, we define $N_{\cS}(T) := \sum_{a \in \cS} N_{a}(T)$ for subset $\cS \subseteq \cA$. 
For any algorithm $\mathsf{Alg}$, we fix an arbitrary subset $\cS_0 \subseteq \cA$ such that $|\cS_0| = A$. For a parameter $\Delta \in (0,1/2]$ to be specified later, define the first bandit instance $\nu_0$ by
\begin{align*}
r_0(a) := \begin{cases}
\Delta, & a \in \cS_0,\\
0, & a \in \cS_0^\complement.
\end{cases}
\end{align*}
Thus, $\nu_0$ has exactly $A$ optimal arms, namely the arms in $\cS_0$. Let $\PP_0$ and $\EE_0$ denote probability and expectation under $\nu_0$ and $\alg$. Since $\sum_{a \in \cS_0^\complement} \EE_0[N_{a}(T)] \leq T$, there exists a subset $\cS_1 \subseteq \cS_0^\complement$ with $|\cS_1| = A$ such that
\begin{align}
\EE_0[N_{\cS_1}(T)] \leq \frac{AT}{K-A}. \label{eq:least-sampled-subset}
\end{align}
Indeed, one may choose $\cS_1$ to consist of the $A$ least-sampled arms in $\cS_0^\complement$. Such a subset exists because $A \leq K/4$, and hence $A < K-A$. We now define a second bandit instance $\nu_1$ by
\begin{align*}
r_1(a) := \begin{cases}
2\Delta, & a \in \cS_1,\\
\Delta, & a \in \cS_0,\\
0, & a \in (\cS_0 \cup \cS_1)^\complement.
\end{cases}
\end{align*}
The instance $\nu_1$ also has exactly $A$ optimal arms, namely the arms in $\cS_1$. Similarly, we use $\PP_1$ and $\EE_1$ to denote probability and expectation under $\nu_1$ and $\alg$, and let $\regret_i(T)$ denote the expected regret under $\nu_i$. Now we consider the event $\cE := \big\{N_{\cS_1}(T) \geq T/2\big\}$. Under $\nu_0$, every arm in $\cS_1$ is suboptimal with gap $\Delta$. Consequently,
\begin{align*}
\regret_0(T) &\geq \Delta \EE_0[N_{\cS_1}(T)] \geq \frac{\Delta T}{2}\PP_0[\cE].
\end{align*}
Under $\nu_1$, every arm outside $\cS_1$ has suboptimality gap at least $\Delta$. Therefore,
\begin{align*}
\regret_1(T) \geq \Delta \EE_1[T-N_{\cS_1}(T)] \geq \frac{\Delta T}{2}\PP_1 \big[\cE^\complement \big].
\end{align*}
Combining the preceding two inequalities gives
\begin{align}
\regret_0(T)+\regret_1(T) \geq \frac{\Delta T}{2}\bigl(\PP_0[\cE]+\PP_1[\cE^\complement]\bigr).
\label{eq:regret-testing-reduction}
\end{align}

It remains to bound the right-hand-side of~\eqref{eq:regret-testing-reduction}. First, by Lemma~\ref{lem:bretagnolle-huber}
\begin{align*}
\PP_0[\cE]+\PP_1[\cE^\complement] \geq \frac{1}{2}\exp\big(-\KL(\PP_0\|\PP_1)\big)
\end{align*}
Since the two instances differ only on the arms in $\cS_1$. For every $a \in \cS_1$, the reward distribution is $\cN(0,1)$ under $\nu_0$ and $\cN(2\Delta,1)$ under $\nu_1$. Using the chain rule for KL divergence under adaptive sampling~\citep[Lemma 15.1]{lattimore2020bandit},
\begin{align*}
\KL(\PP_0 \| \PP_1) = \sum_{a \in \cS_1}\EE_0[N_{a}(T)]\KL\bigl(\cN(0,1) \|\cN(2\Delta,1)\bigr) = 2\Delta^2\EE_0[N_{\cS_1}(T)] \leq \frac{2\Delta^2AT}{K-A},
\end{align*}
where the last inequality follows from \eqref{eq:least-sampled-subset}. Now we choose $\Delta = \frac{1}{4}\sqrt{(K-A)(AT)^{-1}}$, which gives $\KL(\PP_0 \|\PP_1) \leq \frac{1}{8}$. As a result, $\PP_0[\cE]+\PP_1[\cE^\complement] \geq \exp(-1/8)/2$. Substituting this inequality into \eqref{eq:regret-testing-reduction} and taking the max over $\nu_0$ and $\nu_1$ gives
\begin{align*}
\max_{\nu_0, \nu_1} \regret(T) \geq \frac{e^{-1/8}}{8}\Delta T = \Omega\bigg(\sqrt{\frac{(K-A)T}{A}}\bigg) = \Omega\bigg(\sqrt{\frac{KT}{A}}\bigg),
\end{align*}
which finishes the proof.
\end{proof}

\subsection{Proof of Theorem~\ref{thm:lower-many-optimal-arms}}\label{app:proof-lower-many-optimal-arms}

To prove Theorem~\ref{thm:lower-many-optimal-arms}, we need the following change-of-measure argument.

\begin{proposition}
\label{prop:change-of-measure}

Let $(\Omega,\cG)$ be a measurable space, and let $\PP$ and $\QQ$ be probability measures on $(\Omega,\mathcal G)$ satisfying $\PP\ll\QQ$. Define the log-likelihood ratio $L = \log(\ud\PP/\,\ud\QQ) $. Then, for every $\mathcal G$-measurable random variable $X:\Omega\to[0,b]$ and every $\gamma\in\RR$,
\begin{align}
    \EE_{\PP}[X] \leq e^\gamma\EE_{\QQ}[X] +  b\PP[L>\gamma].
    \label{eq:change-of-measure}
\end{align}
Equivalently,
\begin{align}
    \EE_{\QQ}[X]
    \geq
    e^{-\gamma}
    \big[
        \EE_{\PP}[X]
        -
        b\PP(L>\gamma)
    \big].
    \label{eq:change-of-measure-lower}
\end{align}
\end{proposition}

\begin{proof}[Proof of Proposition~\ref{prop:change-of-measure}]
Decompose the expectation according to whether the log-likelihood ratio exceeds $\gamma$:
\begin{align*}
    \EE_{\PP}[X]
    &=
    \EE_{\PP}\big[X\mathbf{1}\{L\leq\gamma\}\big]
    +
    \EE_{\PP}\big[X\mathbf{1}\{L>\gamma\}\big].
\end{align*}
For the first term, the change-of-measure identity gives
\begin{align*}
    \EE_{\PP}\big[X\mathbf{1}\{L\leq\gamma\}\big] =
    \EE_{\QQ}\big[Xe^L\mathbf{1}\{L\leq\gamma\}\big] \leq
    e^\gamma\EE_{\QQ}\big[X\mathbf{1}\{L\leq\gamma\}\big] \leq
    e^\gamma\EE_{\QQ}[X],
\end{align*}
where the last inequality follows from $X\geq 0$. For the second term, since $X\leq b$,
\begin{align*}
    \EE_{\PP}\big[X\mathbf{1}\{L>\gamma\}\big]
    \leq
    b\PP(L>\gamma).
\end{align*}
Combining the two bounds yields
\begin{align*}
    \EE_{\PP}[X]
    \leq
    e^\gamma\EE_{\QQ}[X]
    +
    b\PP(L>\gamma),
\end{align*}
which finishes the proof.
\end{proof}

\begin{remark}\label{rmk:multiplicative-change-of-measure}
When specializing to $X=\ind\{\cE\}$, where $\cE \in \cG$ is any measurable event, Proposition~\ref{prop:change-of-measure} reduces to
$\PP(\cE)\leq e^\gamma\QQ(\cE)+\PP[L>\gamma]$.
This change-of-measure inequality, in its event-based form, was first introduced for establishing instance-dependent lower bounds~\citep{degenne2019pure,garivier2021nonasymptotic}, as an modification to the previous KL-contraction-based technique~\citep{garivier2016optimal,garivier2019explore}. Recently, \citet{ji2026optimal} adopted this change-of-measure technique, together with subsequent combinatorial arguments, to prove the worst case lower bound in offline MABs with multiple optimal arms.
\end{remark}

Next we investigate the specific form of the likelihood ratio between two bandits $\mu$ and $\lambda$ with Gaussian reward noise.

\begin{proposition}[Log-likelihood ratio for Gaussian bandits]
\label{prop:gaussian-bandit-likelihood-ratio}
Fix a bandit algorithm and two bandit instance $\mu$ and $\lambda$ with mean vectors $\blambda,\bmu \in\RR^K$. Define the filtration
\begin{align*}
    &\cF_0:=\{\emptyset,\Omega\}, \qquad \cF_t:=\sigma(a_1,r_1,\ldots,a_t,r_t), \qquad t\in[T],\\
    &\cG_{t+1} := \sigma(\cF_t,a_{t+1}), \qquad t=0,\ldots,T-1.
\end{align*}
Then we see that, under the instance $\lambda$, $r_t \mid \cG_t \sim \mathcal N(\lambda_{a_t},1)$. Let $\PP_{\lambda}$ and $\PP_{\mu}$ denote the trajectory distributions induced by the same algorithm under $\lambda$ and $\mu$, respectively. For every stopping time $\tau\leq T$, define
\begin{align*}
    \PP_{\lambda}^{\tau} := \PP_{\lambda}|_{\cF_\tau},
    \qquad \PP_{\mu}^{\tau} := \PP_{\mu}|_{\cF_\tau}.
\end{align*}
Then the log-likelihood ratio between the two stopped trajectory distributions satisfies
\begin{align}
    L_{\lambda,\mu}(\tau) &:= \log \frac{\ud\PP_{\lambda}^{\tau}}{\ud\PP_{\mu}^{\tau}} = \sum_{t=1}^{\tau} \bigg[(\lambda_{a_t}-\mu_{a_t}) (r_t-\lambda_{a_t}) + \frac{1}{2}(\lambda_{a_t}-\mu_{a_t})^2 \bigg].
    \label{eq:gaussian-bandit-llr-numerator}
\end{align}
\end{proposition}

\begin{proof}
For every $t\in[T]$, let $\pi_t(\cdot\mid h_{t-1})$ denote the conditional distribution used by the algorithm to select $a_t$ given the history $h_{t-1} = (a_1,r_1,\ldots,a_{t-1},r_{t-1})$.  Since the same algorithm is used under both instances, the action-selection kernel $\pi_t$ is identical under $\PP_{\lambda}$ and $\PP_{\mu}$. In particular, it does not depend on the underlying bandit instance.

Let $\phi(x) := (\sqrt{2\pi})^{-1} \exp(-x^2/2)$ denote the standard Gaussian density. Fix any deterministic $n\in[T]$. With respect to the product of the counting measure on actions and the Lebesgue measure on rewards, the density of the trajectory up to time $n$ under $\lambda$ and $\mu$ are
\begin{align*}
    & p_{\lambda}(h_n) = \prod_{t=1}^n \pi_t(a_t\mid h_{t-1}) \phi(r_t-\lambda_{a_t}), \\
    & p_{\mu}(h_n) = \prod_{t=1}^n \pi_t(a_t\mid h_{t-1}) \phi(r_t-\mu_{a_t}).
\end{align*}
Therefore, the action-selection probabilities cancel, and we obtain
\begin{align}
    \frac{
        \ud\left(\PP_{\lambda}|_{\cF_n}\right)
    }{
        \ud\left(\PP_{\mu}|_{\cF_n}\right)
    }
    &=
    \prod_{t=1}^n
    \frac{
        \phi(r_t-\lambda_{a_t})
    }{
        \phi(r_t-\mu_{a_t})
    }.
    \label{eq:gaussian-bandit-likelihood-ratio-deterministic}
\end{align}
Define the likelihood-ratio process $Z_n = \ud \PP_{\lambda}|_{\cF_n} / \ud \PP_{\mu}|_{\cF_n}$ for all $n \in [T]$, and set $Z_0:=1$. We next show that $Z_{\tau}$ is the
Radon--Nikodym derivative between the two trajectory distributions
restricted to $\cF_{\tau}$. Indeed, for every event
$E\in\cF_{\tau}$, we have
\begin{align*}
    \EE_{\mu}\big[Z_{\tau}\ind\{E\}\big] =
    \sum_{n=0}^T  \EE_{\mu}\big[ Z_n\ind\{E\cap\{\tau=n\}\} \big] = 
    \sum_{n=0}^T
    \PP_{\lambda}\big(E\cap\{\tau=n\}\big) = \PP_{\lambda}(E),
\end{align*}
where the second equality follows because
$E\cap\{\tau=n\}\in\cF_n$. Consequently, $Z_\tau = \ud \PP_{\lambda}^{\tau} / \ud \PP_{\mu}^{\tau}$.

It remains to compute the logarithm of $Z_{\tau}$. For every
$t\in[T]$,
\begin{align*}
    \log
    \frac{
        \phi(r_t-\lambda_{a_t})
    }{
        \phi(r_t-\mu_{a_t})
    }
    = 
    \frac{1}{2}
    \big[
        (r_t-\mu_{a_t})^2
        -
        (r_t-\lambda_{a_t})^2
    \big] = 
    (\lambda_{a_t}-\mu_{a_t})
    (r_t-\lambda_{a_t})
    +
    \frac{1}{2}
    (\lambda_{a_t}-\mu_{a_t})^2.
\end{align*}
Taking the logarithm of
\eqref{eq:gaussian-bandit-likelihood-ratio-deterministic} and
evaluating the resulting likelihood-ratio process at $\tau$ therefore
gives
\begin{align*}
    L_{\lambda,\mu}(\tau)
    =
    \log
    \frac{
        \ud\PP_{\lambda}^{\tau}
    }{
        \ud\PP_{\mu}^{\tau}
    }
    =
    \sum_{t=1}^{\tau} \bigg[ (\lambda_{a_t}-\mu_{a_t}) (r_t-\lambda_{a_t}) + \frac{1}{2}(\lambda_{a_t}-\mu_{a_t})^2\bigg],
\end{align*}
which proves the proposition.
\end{proof}

Now we are ready to prove Theorem~\ref{thm:lower-many-optimal-arms}.

\begin{proof}[Proof of Theorem~\ref{thm:lower-many-optimal-arms}]
The hard instance construction is inspired by Exercise 15.2 of~\citet{lattimore2020bandit}. We consider the multi-armed bandit class parameterized by mean vectors:
\begin{align*}
\mu \in \cV_A:= \bigg\{ \bmu \in \bigg\{\frac{1}{2} , \frac{1}{2} -\Delta\bigg\}^K \bigg| \sum_{i=1}^K \ind(\mu_i= 1 / 2) = A \bigg\},
\end{align*}
where $\Delta > 0$ is a parameter to be determined later. We further consider an auxiliary instance $\theta$ such that $\btheta = (1 / 2, \cdots, 1/2)^\top \in \RR^K$, and let $\cV = \cV_A \cup \{\theta\}$. Unless stated otherwise, we assume that noise of rewards is standard Gaussian. For any $\mu \in \cV$, the corresponding instance is given by $([K], r_{\mu}, T)$ where the reward is given by $r_{\mu}(i) = \mu_i$.
For any $\mu \in \cV_A$, we use $\aopt(\mu)$ to denote the set of optimal actions (those with reward $1/2$) and $\asubopt(\mu)$ to denote the set of sub-optimal arms. Our goal is to show that, for any algorithm $\alg$, there exists $\mu \in \cV_A$ such that that 
\begin{align*}
    \regret_{\mu, \alg}(T) =\tilde{\Omega}\bigg(\frac{K-A}{K}\sqrt{T}\bigg).
\end{align*}
In the subsequent proof, we will drop the subscript of $\alg$ when there is no ambiguity.

Now we consider any $\mu \in \cV_A$ and its difference between $\theta$, in terms of the distribution of induced trajectories. Consider the random variable $X_{\mu} = \sum_{a \in \asubopt(\mu)} N_{a}(T)$, then we know that 
\begin{align*}
    \regret_{\mu}(T) = \EE_{\mu}[X_{\mu}]\Delta.
\end{align*}
On the other side, since $X_{\mu} \leq T$, Proposition~\ref{prop:change-of-measure} tells us that for any $\gamma > 0$,
\begin{align}
    \EE_{\mu}[X_{\mu}] \geq e^{-\gamma} \bigg[ \EE_{\theta}[X_{\mu}] - T \PP_{\theta}\bigg(\log \frac{\ud\PP_{\theta}}{\ud\PP_{\mu}}  > \gamma \bigg)\bigg]. \label{eq:lower-bound-change-of-measure}
\end{align}
To lower bound the right hand side, we need to properly choose the threshold $\gamma$, which requires the computation of $\log (\ud\PP_{\theta} / \ud\PP_{\mu})$. Applying Proposition~\ref{prop:gaussian-bandit-likelihood-ratio} with $\tau=T$ gives that 
\begin{align*}
    \log \frac{\ud\PP_{\theta}}{\ud\PP_{\mu}} &= \sum_{t=1}^T
    \bigg[ (\theta_{a_t}-\mu_{a_t})(r_t-\theta_{a_t}) + \frac{1}{2}
        (\theta_{a_t}-\mu_{a_t})^2 \bigg] \\
    & = \underbrace{\frac{1}{2}\sum_{a \in \asubopt(\mu)} N_{a}(T)(\theta_a - \mu_a)^2}_{I_1} + \underbrace{\sum_{t=1}^T \bigg(r_t-\frac12\bigg) \Delta\ind\{a_t \in \asubopt(\mu)\}}_{I_2},
\end{align*}
Since $\sum_{a \in \asubopt(\mu)} N_{a}(T) \leq T$, we know that $I_1 \leq T\Delta^2/2$. For the second term, we consider
\begin{align*}
    M_t = \sum_{s=1}^t \bigg(r_s - \frac12 \bigg) \Delta\ind\{\cE_s(\mu)\}, \quad V_t = \sum_{s=1}^t \Delta^2 \ind\{\cE_s(\mu)\}, \text{ where } \cE_s(\mu) = \{a_s \in \asubopt(\mu)\}
\end{align*}
We further define
\begin{align*}
    Z_t = \exp\bigg(\lambda M_t-\frac{\lambda^2V_t}{2} \bigg), \quad t=0,1,\ldots,T.
\end{align*}
Then under the probability distribution $\PP_\theta$, we see that for any $\lambda \in \RR$,
\begin{align*}
    & \EE_{\theta}[Z_t | \cF_{t-1}] = Z_{t-1} \EE_{\theta}\bigg[\underbrace{\exp\bigg(\lambda (r_t - 1/2) \Delta\ind\{\cE_t(\mu)\}-\frac{\lambda^2\Delta^2 \ind\{\cE_t(\mu)\}}{2} \bigg)}_{(\star)} \bigg| \cF_{t-1}\bigg] 
\end{align*}
For $(\star)$, we first take expectation on $\cG_t$, which gives that 
\begin{align*}
    \EE_{\theta}[(\star) | \cG_t] &  =1 - \ind\{\cE_t(\mu)\} +  \ind\{\cE_t(\mu)\} \EE_{\theta}\bigg[\exp\bigg(\lambda (r_t - 1/2) \Delta -\frac{\lambda^2\Delta^2 }{2} \bigg) \bigg| \cG_t \bigg] \\
    & \leq 1 - \ind\{\cE_t(\mu)\} + \ind\{\cE_t(\mu)\} \\
    & = 1,
\end{align*}
where the inequality holds due to that $r_t - 1/2$ is $1$-sub-gaussian conditioned on $\cG_t$ on $\PP_\theta$. This immediately gives $\EE_{\theta}[(\star) | \cF_{t-1}] \leq 1$, verifying that $Z_t$ is a super-martingale for any $\lambda$. Now invoking Lemma~\ref{lem:sub-gaussian-super-martingale} with $\tau=T$, since $I_2 = M_T$, we see that for any $\beta > 0$,
\begin{align*}
    \PP_{\theta}(I_2 \ge \beta) \leq \exp\bigg(-\frac{\beta^2}{2T\Delta^2}\bigg).
\end{align*}
Now combining the estimation of $I_1$ and $I_2$ and set $\beta = 2 \Delta \sqrt{T\log K}$, and $\gamma = \beta + T\Delta^2 /2$, we obtain that 
\begin{align*}
    \PP_{\theta}\bigg(\log \frac{\ud\PP_{\theta}}{\ud\PP_{\mu}}  > \gamma \bigg) \leq \exp\bigg(-\frac{\beta^2}{2T\Delta^2}\bigg) = \frac{1}{K^2}.
\end{align*}
Plugging into~\eqref{eq:lower-bound-change-of-measure}, we see that
\begin{align*}
    \EE_{\mu}[X_{\mu}] & \geq \exp \bigg(-\beta-\frac{T\Delta^2}{2}\bigg)   \bigg[ \EE_{\theta}[X_{\mu}] - T \PP_{\theta}\bigg(\log \frac{\ud\PP_{\theta}}{\ud\PP_{\mu}}  > \gamma \bigg)\bigg] \\
    & \geq \exp \bigg(-\beta-\frac{T\Delta^2}{2}\bigg)   \bigg[ \EE_{\theta}[X_{\mu}] - \frac{T}{K^2}\bigg].
\end{align*}
Now we sum up over all $\mu \in \cV_A$ and obtains that
\begin{align*}
    \sum_{\mu \in \cV_A} \EE_{\mu}[X_{\mu}] & \geq \exp \Bigg(-\beta-\frac{T\Delta^2}{2}\bigg) \bigg(\sum_{\mu \in \cV_A} \EE_{\theta}[X_{\mu}]  - \sum_{\mu \in \cV_A} \frac{T}{K^2} \Bigg) \\
    & = \exp \bigg(-\beta-\frac{T\Delta^2}{2}\bigg) \Bigg(\sum_{\mu \in \cV_A} \sum_{a \in \asubopt(\mu)} \EE_{\theta}[N_{a}(T)] - \binom{K}{K-A}\frac{T}{K^2} \Bigg) \\
    & = \exp \bigg(-\beta-\frac{T\Delta^2}{2}\bigg) \Bigg(\sum_{a \in \cA} \sum_{\mu:a \in \asubopt(\mu)}  \EE_{\theta}[N_{a}(T)] - \binom{K}{K-A}\frac{T}{K^2} \Bigg) \\
    & = \exp \bigg(-\beta-\frac{T\Delta^2}{2}\bigg) \Bigg(T \binom{K-1}{K-A-1}   - \binom{K}{K-A}\frac{T}{K^2} \Bigg),
\end{align*}
where the second equation is given by $|\cV_A| = \binom{K}{K-A}$ and the last equation holds due to $\sum_{a \in \cA} N_{a}(T) = T$ and $|\{\mu:a \in \asubopt(\mu)\}| = \binom{K-1}{K-A-1}$. Now take maximum over all $\mu \in \cV_A$ and leverage the fact that maximum is larger that average, we see that
\begin{align*}
    \sup_{\mu \in \cV_A} \EE_{\mu}[X_{\mu}] & \geq T\exp \bigg(-\beta-\frac{T\Delta^2}{2}\bigg) \bigg[\binom{K}{K-A}^{-1}\binom{K-1}{K-A-1} - \frac{1}{K^2} \bigg] \\
    & \geq T \exp \bigg(-\beta-\frac{T\Delta^2}{2}\bigg) \bigg(\frac{K-A}{K} - \frac{1}{K^2}\bigg)  \\
    & \geq \frac{(K-A)T}{2K} \exp \bigg(-\beta-\frac{T\Delta^2}{2}\bigg). 
\end{align*}
Recall that $\regret_{\mu}(T) = \EE_{\mu}[X_{\mu}]\Delta$, we see
\begin{align*}
    \sup_{\mu \in \cV_A} \regret_{\mu}(T) \geq \frac{(K-A)T\Delta}{2K} \exp \bigg(-\beta-\frac{T\Delta^2}{2}\bigg).
\end{align*}
Finally, recall that $\beta = 2 \Delta \sqrt{T\log K}$. Now we select $\Delta = \sqrt{(T\log K)^{-1}}$, we see that $\beta = 2$ and $\gamma = 2 + (2\log K)^{-1} \leq 3$. This leads to
\begin{align*}
    \sup_{\mu \in \cV_A} \regret_{\mu}(T) \gtrsim \frac{(K-A)T}{K\sqrt{T\log K}} = \tilde{\Omega}\bigg(\frac{K-A}{K}\sqrt{T}\bigg),
\end{align*}
for any algorithm $\alg$, which finishes the proof.
\end{proof}

\section{Proofs of Theorems in Section~\ref{sec:non-adaptivity}}\label{app:proof-non-adaptivity}

\subsection{Formal Version and Proof of Lemma~\ref{lem:exploration-profile-upper-bound}}

In this section, we provide the formal version of Lemma~\ref{lem:exploration-profile-upper-bound} and the proof of it.

\begin{lemma}\label{lem:pull-count-tail}
For some $1\leq A<K$, consider an algorithm $\alg$ satisfying
\begin{align*}
    \sup_{\nu:\,|\aopt(\nu)|=A}
    \regret_{\nu,\alg}(T)
    \leq
    \Gamma\frac{K-A}{K}\sqrt{T}
\end{align*}
for some $\Gamma\geq1$. Let $\theta$ denote the all-equal instance, then, for every $m\in[T]$ and every $\ell\geq1$,
\begin{align*}
    \max_{\cS\subseteq\cA:\,|\cS|=K-A} \PP_{\theta}\big[N_{\cS}(T)\geq m\big] & \geq \frac{K-A}{K} - \exp(-\ell) - 8\Gamma\frac{K-A}{K}\sqrt{\frac{m\ell}{T}} \\
    \max_{\cS\subseteq\cA:\,|\cS|=K-A} \PP_{\theta}\big[N_{\cS}(T)\geq m\big] 
    &\leq \exp(-\ell) + 8\Gamma\frac{K-A}{K}\sqrt{\frac{T\ell}{m}}.
\end{align*}
\end{lemma}

\begin{proof}[Proof of Lemma~\ref{lem:pull-count-tail}]
Fix any $m\in[T]$ and $\ell\geq1$, and for simplicity we write
\begin{align*}
    q_{\alg}(m):= \max_{\cS\subseteq\cA:\,|\cS|=K-A}
    \PP_\theta\big(N_{\cS}(T)\geq m\big).
\end{align*}
Recall that $\theta$ is the all-equal instance with mean vector $\btheta=\mathbf 1/2$. For every subset $\cS\subseteq\cA$ with $|\cS|=K-A$, construct an
alternative Gaussian bandit instance $\nu_{\cS}$ with mean vector
$\bnu_{\cS}$ given by
\begin{align*}
    \nu_{\cS,a}
    :=
    \begin{cases}
        1/2-\Delta, & a\in\cS,\\
        1/2, & a\notin\cS,
    \end{cases}
    \qquad
    \Delta
    :=
    \frac{1}{4\sqrt{m\ell}}.
\end{align*}
Since $|\cS| = K-A$, the instance $\nu_{\cS}$ has exactly $A$
optimal arms. Moreover, $\Delta\leq1/4$, so all its reward means
belong to $[0,1]$. We show that, whenever the algorithm pulls only $m$ or fewer than $m$ times on the set $\cS$, it cannot reliably distinguish $\theta$ and $\nu_{\cS}$. To show this, for a fixed $\cS$, we define the stopping time $\tau_{\cS}$ at which the pull count on $\cS$ is exactly $m$
\begin{align*}
    H_{\cS} := \{N_{\cS}(T)\geq m\},\quad 
    \tau_{\cS} := \inf\{t\in[T]:N_{\cS}(t)=m\}\wedge T.
\end{align*}
Then $\tau_{\cS}$ is a bounded $(\cF_t)_{t=0}^T$-stopping time, and
\begin{align}
    H_{\cS} = \{N_{\cS}(\tau_{\cS})=m\} \in  \cF_{\tau_{\cS}}.
    \label{eq:pull-count-event-stopped-measurable}
\end{align}

Now we show that, up to the stopping time $\tau_{\cS}$, the distribution given by $\theta$ and $\nu_{\cS}$ is close. We once again characterize this closeness via the log-likelihood ratio. In particular, consider the
trajectory distributions restricted to $\cF_{\tau_{\cS}}$ by $\PP_\theta^{\tau_{\cS}} := \PP_\theta|_{\cF_{\tau_{\cS}}}$ and $\PP_{\nu_{\cS}}^{\tau_{\cS}} := \PP_{\nu_{\cS}}|_{\cF_{\tau_{\cS}}}$. Applying Proposition~\ref{prop:gaussian-bandit-likelihood-ratio} gives that
\begin{align*}
    \log \frac{\ud\PP_\theta^{\tau_{\cS}}}{\ud\PP_{\nu_{\cS}}^{\tau_{\cS}}} & = \sum_{t=1}^{\tau_{\cS}} \bigg[ (\theta_{a_t}-\nu_{\cS, a_t}) (r_t-\theta_{a_t}) + \frac{1}{2}(\theta_{a_t}-\nu_{\cS, a_t})^2\bigg] \\
    & = \Delta \sum_{t=1}^{\tau_{\cS}} \ind\{a_t\in\cS\}(r_t-1/2) +
    \frac{\Delta^2}{2}
    N_{\cS}(\tau_{\cS}).
\end{align*}
Now we denote $M_t = \Delta \sum_{s=1}^t \ind\{a_s\in\cS\}(r_s-1/2)$ and $V_t = \Delta^2 N_{\cS}(t)$
, then under the probability distribution $\PP_\theta$, we see that for any $\lambda \in \RR$,
\begin{align*}
    & \EE_{\theta}\bigg[\exp\bigg(\lambda M_t-\frac{\lambda^2V_t}{2} \bigg) \bigg| \cF_{t-1}\bigg] \\
    & \quad = \exp\bigg(\lambda M_{t-1}-\frac{\lambda^2V_{t-1}}{2} \bigg) \EE_{\theta}\bigg[\underbrace{\exp\bigg(\lambda (r_t-1/2) \Delta\ind\{a_t\in\cS\}-\frac{\lambda^2\Delta^2 \ind\{a_t\in\cS\}}{2} \bigg)}_{(\star)} \bigg| \cF_{t-1}\bigg] 
\end{align*}
For $(\star)$, we first take expectation on $\cG_t$, which gives that
\begin{align*}
    \EE_{\theta}[(\star) | \cG_t] &  =1 - \ind\{a_t\in\cS\} +  \ind\{a_t\in\cS\} \EE_{\theta}\bigg[\exp\bigg(\lambda (r_t-1/2) \Delta -\frac{\lambda^2\Delta^2 }{2} \bigg) \bigg| \cG_t \bigg] \\
    & \leq 1 - \ind\{a_t\in\cS\} + \ind\{a_t\in\cS\} \\
    & = 1,
\end{align*}
where the inequality holds due to that $r_t-1/2$ is $1$-sub-gaussian conditioned on $\cG_t$ on $\PP_\theta$. This immediately gives $\EE_{\theta}[(\star) | \cF_{t-1}] \leq 1$, verifying that $\exp(\lambda M_t - \lambda^2 V_t /2)$ is a super-martingale for any $\lambda$. Now invoking Lemma~\ref{lem:sub-gaussian-super-martingale} with $x = \Delta\sqrt{2m\ell}$, since $V_\tau \leq m\Delta^2$, we have
\begin{align*}
    \PP_{\theta}\big(M_\tau \ge \Delta\sqrt{2m\ell}\big) \leq \exp\bigg(-\frac{2m\Delta^2\ell}{2m\Delta^2}\bigg) = \exp(-\ell).
\end{align*}
Since $\Delta^2 N_{\cS}(\tau_{\cS}) /2 \leq m\Delta^2/2$, this immediately leads to
\begin{align*}
    \PP \bigg(\log \frac{\ud\PP_\theta^{\tau_{\cS}}}{\ud\PP_{\nu_{\cS}}^{\tau_{\cS}}} \geq \Delta\sqrt{2m\ell} + m\Delta^2/2 \bigg) \leq \exp(-\ell).
\end{align*}
Recall that $\Delta = (4\sqrt{m\ell})^{-1}$, this means that $\Delta\sqrt{2m\ell} + m\Delta^2/2 = 1/2\sqrt{2} + 1/32\ell < \log 2$ since $\ell \geq 1$
\begin{align}
    \PP \bigg(\log \frac{\ud\PP_\theta^{\tau_{\cS}}}{\ud\PP_{\nu_{\cS}}^{\tau_{\cS}}} > \log 2 \bigg) \leq \exp(-\ell).
    \label{eq:pull-count-tail-llr}
\end{align}
We now prove the upper bound on $q_{\alg}(m)$. Fix any $\cS\subseteq\cA$ with $|\cS| = K-A$. Since $H_{\cS}\in\cF_{\tau_{\cS}}$, we may apply
Proposition~\ref{prop:change-of-measure} on the measurable space
$(\Omega,\cF_{\tau_{\cS}})$ with $\PP= \PP_\theta^{\tau_{\cS}}$, $\QQ =   \PP_{\nu_{\cS}}^{\tau_{\cS}}$, $X = \ind\{H_{\cS}\}$, $b=1$ and $\gamma = \log 2$. Using \eqref{eq:pull-count-tail-llr}, we obtain
\begin{align}
    \PP_{\nu_{\cS}}(H_{\cS}) \geq \frac{1}{2} \big[\PP_\theta(H_{\cS}) -  \PP_\theta(L_{\cS}>\log 2) \big] \geq \frac{1}{2} \big[ \PP_\theta(H_{\cS})  -\exp(-\ell) \big].
    \label{eq:pull-count-tail-event-com}
\end{align}
Every arm in $\cS$ is suboptimal under $\nu_{\cS}$ with gap
$\Delta$. Hence, the assumed regret guarantee gives
\begin{align*}
    \Gamma\frac{K-A}{K}\sqrt{T} \geq \regret_{\nu_{\cS},\alg}(T)= \Delta
    \EE_{\nu_{\cS}}[N_{\cS}(T)] &\geq \Delta m\PP_{\nu_{\cS}}(H_{\cS})  \geq \frac{\Delta m}{2} \big[ \PP_\theta(H_{\cS}) - \exp(-\ell)\big].
\end{align*}
Rearranging and substituting $\Delta=(4\sqrt{m\ell})^{-1}$ yields
\begin{align*}
    \PP_\theta\big(N_{\cS}(T)\geq m\big) \leq  \exp(-\ell)  + \frac{2\Gamma(K-A)\sqrt{T}}{K\Delta m} =  \exp(-\ell) + 8\Gamma\frac{K-A}{K}    \sqrt{\frac{T\ell}{m}}.
\end{align*}
Since this holds for every $|\cS|=K-A$, taking maximum over all $|\cS|=K-A$ gives the desired upper bound
\begin{align}
    q_{\alg}(m) \leq \exp(-\ell) +  8\Gamma\frac{K-A}{K} \sqrt{\frac{T\ell}{m}}.
    \label{eq:pull-count-tail-upper}
\end{align}

We next prove the lower bound on $q_{\alg}(m)$. For every
$\cS\subseteq\cA$ with $|\cS|=K-A$, define $Y_{\cS} = N_{\cS}(T) \ind \{ N_{\cS}(T) < m\} = N_{\cS}(T)\ind\{H_{\cS}^{\complement}\}$. On $H_{\cS}^{\complement}$, the stopping time satisfies $\tau_{\cS}=T$, whereas on $H_{\cS}$ both sides below vanish. Therefore, $Y_{\cS} = N_{\cS}(\tau_{\cS})\ind\{H_{\cS}^{\complement}\}$. In particular, $Y_{\cS}$ is $\cF_{\tau_{\cS}}$-measurable and takes values in $[0,T]$. Once again, applying Proposition~\ref{prop:change-of-measure} once more on $(\Omega,\cF_{\tau_{\cS}})$, now with $\PP= \PP_\theta^{\tau_{\cS}}$, $\QQ =   \PP_{\nu_{\cS}}^{\tau_{\cS}}$, $X = Y_{\cS}$, $b=T$ and $\gamma = \log 2$, and using \eqref{eq:pull-count-tail-llr}, gives
\begin{align*}
    \EE_{\nu_{\cS}}[Y_{\cS}]  \geq \frac{1}{2} \big[ \EE_\theta[Y_{\cS}] - T\PP_\theta(L_{\cS}>\log 2) \big] \geq \frac{1}{2} \big[ \EE_\theta[Y_{\cS}] - T\exp(-\ell)\big].
\end{align*}
Since $Y_{\cS}\leq N_{\cS}(T)$, we have
\begin{align*}
    \EE_{\nu_{\cS}}[N_{\cS}(T)] \geq \EE_{\nu_{\cS}} [Y_{\cS}] \geq \frac{1}{2} \big[ \EE_\theta \big[ N_{\cS}(T)\ind\{H_{\cS}^{\complement}\} \big] - T\exp(-\ell) \big].
\end{align*}
Moreover, as $\PP_\theta(H_{\cS})\leq q_{\alg}(m)$, it follows that
\begin{align*}
    \EE_{\nu_{\cS}}[N_{\cS}(T)] & \geq \frac{1}{2} \big[ \EE_\theta \big[ N_{\cS}(T)\ind\{H_{\cS}^{\complement}\} \big] - T\exp(-\ell) \big] \\
    &\geq  \frac{1}{2} \big[\EE_\theta[N_{\cS}(T)] - T\PP_\theta(H_{\cS}) - T\exp(-\ell)  \big]\\
    &\geq \frac{1}{2}\EE_\theta[N_{\cS}(T)] - \frac{T}{2}q_{\alg}(m) - \frac{T}{2}\exp(-\ell).
\end{align*}
Thus,
\begin{align}
    \EE_{\nu_{\cS}}[N_{\cS}(T)]
    \geq
    \frac{1}{2}\EE_\theta[N_{\cS}(T)]
    -
    \frac{T}{2}q_{\alg}(m)
    -
    \frac{T}{2}\exp(-\ell).
    \label{eq:pull-count-tail-single-set-lower}
\end{align}

We now average
\eqref{eq:pull-count-tail-single-set-lower} over all subsets
$\cS\subseteq\cA$ with $|\cS|=K-A$. Since every arm belongs to
$\binom{K-1}{K-A-1}$ such subsets,
\begin{align*}
    \frac{1}{\binom{K}{K-A}} \sum_{\cS\subseteq\cA:\,|\cS|=K-A}\EE_\theta[N_{\cS}(T)] = \frac{\binom{K-1}{K-A-1}}{\binom{K}{K-A}} \sum_{a\in\cA}\EE_\theta[N_{a}(T)] = \frac{K-A}{K}T.
\end{align*}
Therefore,
\begin{align}
    \frac{1}{\binom{K}{K-A}}  \sum_{\cS\subseteq\cA:\,|\cS|=K-A}  \EE_{\nu_{\cS}}[N_{\cS}(T)] \geq \frac{T}{2} \bigg[ \frac{K-A}{K} - q_{\alg}(m) - \exp(-\ell)\bigg].
    \label{eq:pull-count-tail-averaged-lower}
\end{align}
Since 
Every $\nu_{\cS}$ has exactly $A$ optimal arms, we have $\regret_{\nu_{\cS},\alg}(T) = \Delta \EE_{\nu_{\cS}}[N_{\cS}(T)]$ . Together with \eqref{eq:pull-count-tail-averaged-lower}, this implies that
\begin{align*}
    \Gamma\frac{K-A}{K}\sqrt{T} \geq  \frac{\Delta}{\binom{K}{K-A}} \sum_{\cS\subseteq\cA:\,|\cS|=K-A} \EE_{\nu_{\cS}} [N_{\cS}(T)] 
    \geq \frac{\Delta T}{2} \bigg[\frac{K-A}{K} -  q_{\alg}(m) - \exp(-\ell)\bigg].
\end{align*}
Rearranging and substituting
$\Delta=(4\sqrt{m\ell})^{-1}$ gives
\begin{align}
    q_{\alg}(m) &\geq \frac{K-A}{K} - \exp(-\ell) - \frac{2\Gamma(K-A)\sqrt{T}}{K\Delta T} \\
    &= \frac{K-A}{K} - \exp(-\ell) - 8\Gamma\frac{K-A}{K}\sqrt{\frac{m\ell}{T}}.
    \label{eq:pull-count-tail-lower}
\end{align}

Combining
\eqref{eq:pull-count-tail-upper} and
\eqref{eq:pull-count-tail-lower} proves
\begin{align*}
    \frac{K-A}{K}- \exp(-\ell) -     8\Gamma\frac{K-A}{K}  \sqrt{\frac{m\ell}{T}} \leq q_{\alg}(m) \leq  \exp(-\ell)  + 8\Gamma\frac{K-A}{K}\sqrt{\frac{T\ell}{m}},
\end{align*}
as desired.
\end{proof}

\subsection{Tightness of Lemma~\ref{lem:pull-count-tail}}\label{app:tail-envelop-tightness}

In this section, we show that both inequalities in Lemma~\ref{lem:pull-count-tail} cannot be improved in general, for all $m = \tilde{O}(T)$ and $A \geq K/2$ and $\ell$ chosen to be the same order as the remaining factors in the bound. For simplicity, given any algorithm $\alg$, we write
\begin{align*}
    q_{\alg}(m):= \max_{\cS\subseteq\cA:\,|\cS|=K-A}
    \PP_\theta\big(N_{\cS}(T)\geq m\big).
\end{align*}
Then a properly chosen $\ell$, together with $m = \tilde{O}(T)$ and $A \geq K/2$ reduces the lower envelop in Lemma~\ref{lem:pull-count-tail} to
\begin{align*}
    q_{\alg}(m) =\tilde{\Omega}\bigg(\frac{K-A}{K}\bigg),
\end{align*}
and the upper envelop to 
\begin{align*}
    q_{\alg}(m) =  \tilde{O}\bigg(\frac{K-A}{K}\sqrt{\frac{T}{m}}\bigg).
\end{align*}
We will show that there exist two algorithms that attain the reduced lower and upper envelop respectively.

\paragraph{Tightness of the Lower Envelope.}
We show that the lower bound is attained by the  near-minimax optimal algorithm $\algname$. Recall that $\algname$ first samples a subset $\cB \subseteq\cA$ of $L$ arms uniformly without replacement and then runs a bandit algorithm only on the arms in $\cB$ for all $T$ rounds, where $L = \tilde{\Theta}(1)$. We show that, under the all-equal instance $\theta$, for every $m\leq T/L$, $q_{\algname}(m) \leq L(K-A)/K$. 

Fix any $\cS\subseteq\cA$ with $|\cS|=K-A$. Since $\algname$ only pulls arms in the randomly sampled subset $\cB$, the event $N_{\cS}(T)\geq m$ for any $m\geq1$ implies that $\cB\cap\cS\neq\emptyset$. Therefore,
\begin{align*}
    \PP_{\btheta}\big[N_{\cS}(T)\geq m\big] \leq     \PP\big[\cU\cap\cS\neq\emptyset\big] \leq    \sum_{a\in\cS}\PP[a\in\cU] = \frac{L(K-A)}{K},
\end{align*}
which gives that
\begin{align}
    q_{\algname}(m) \leq \frac{L(K-A)}{K}.
    \label{eq:tail-envelope-lower-tight-upper}
\end{align}

\begin{algorithm}[t]
\caption{Tail-Saturating Algorithm for Bandit with Multiple Optimal Arms}
\label{algorithm:tail-upper-tight}
    \begin{algorithmic}[1]
    \REQUIRE Action set $\cA$, number of total rounds $T$, number of optimal arms $A$, and threshold $m\leq T/2$.
    \STATE Fix any subset $\cU\subseteq\cA$ with $|\cU|=2(K-A)$ together with a partition $\cU=\cS_0\cup\cS_1$, where $|\cS_0|=|\cS_1|=K-A$.
    \STATE Set
    \begin{align*}
        p_m := \min\bigg\{1,\,       \frac{K-A}{K}\sqrt{\frac{T}{m}}
        \bigg\},
    \end{align*}
    and sample $Z\sim\operatorname{Bernoulli}(p_m)$.
    \IF{$Z=0$}
        \STATE Run $\algname$ on $\cA$ for $T$ rounds.
    \ELSE
        \STATE Sample a uniformly random permutation $\sigma$ of the arms in $\cU$.
        \STATE Run $\algname$ on $\sigma(\cU)$ for $2m$ rounds, using $K'=2(K-A)$ and $A'=K-A$ as its input parameters, and map the selected actions back to $\cU$ through $\sigma^{-1}$.
        \STATE Restart $\algname$ on the full action set $\cA$ for the remaining $T-2m$ rounds.
    \ENDIF
    \end{algorithmic}
\end{algorithm}

\paragraph{Tightness of the upper envelope.}
For every prescribed scale $m\leq T/2$, we construct a minimax-optimal algorithm satisfying
\begin{align*}
q_{\alg}(m) =\widetilde{\Omega}\bigg(\min\bigg\{1,\, \frac{K-A}{K}\sqrt{\frac{T}{m}}\bigg\} \bigg),
\end{align*}
which matches the upper envelope in Lemma~\ref{lem:pull-count-tail}. The constructed algorithm is described in Algorithm~\ref{algorithm:tail-upper-tight}, which is built on the algorithm $\algname$. 

We first show that Algorithm~\ref{algorithm:tail-upper-tight} is minimax-optimal. On the event $Z=1$, we fix any fixed subset $|\cU| = 2(K-A)$. Since the original instance contains only $K-A$ suboptimal arms in total, the set $\cU$ of $2(K-A)$ arms contains at least $K-A$ optimal arms. Therefore, when $\algname$ is run on $\cU$ with parameters $K'=2(K-A)$ and $A'=K-A$, the same subsampling analysis as in Section~\ref{sec:upper-bound} gives
\begin{align}
    \sup_{\nu:\,|\aopt(\nu)|=A} \regret_{\nu,\algname}^{\cU}(2m)  =  \widetilde{O}(\sqrt{m}),
    \label{eq:tail-upper-probe-regret}
\end{align}
where $\regret_{\nu,\algname}^{\cU}(2m)$ denotes the regret accumulated during the exploration phase on $\cU$. The remainder of Algorithm~\ref{algorithm:tail-upper-tight} simply runs $\algname$ on the original problem, and hence contributes at most $\widetilde{O}((K-A)/K\sqrt{T})$ regret. Consequently,
\begin{align*}
    \sup_{\nu:\,|\aopt(\nu)|=A} \regret_{\nu,\alg}(T) \leq  \widetilde{O}\bigg(\frac{K-A}{K}\sqrt{T}\bigg) +p_m\widetilde{O} (\sqrt{m}) = \widetilde{O}\bigg(\frac{K-A}{K}\sqrt{T}\bigg),
\end{align*}
where the last equality follows from
\begin{align*}
    p_m\sqrt{m}=\min\bigg\{\sqrt{m},\,\frac{K-A}{K}\sqrt{T} \bigg\} \leq\frac{K-A}{K}\sqrt{T}.
\end{align*}
Thus Algorithm~\ref{algorithm:tail-upper-tight} remains minimax-optimal up to logarithmic factors.

It remains to characterize its pull-count tail under the all-equal instance $\theta$. Conditional on $Z=1$, consider only the first $2m$ rounds. Let $Y:=N_{\cS_0}(2m)$. Since every arm has the same reward distribution under $\theta$, and we apply a uniformly random permutation to $\cU$ before running the exploration algorithm, the resulting trajectory distribution is invariant under permutations of the arms in $\cU$. In particular, swapping $\cS_0$ and $\cS_1$ leaves the distribution unchanged. Moreover, every action during the first $2m$ rounds belongs to $\cU=\cS_0\cup\cS_1$. These leads to 
\begin{align*}
    Y = 2m - N_{\cS_1}(2m) \overset{d}{=} 2m - Y,
\end{align*}
which immediately gives
\begin{align*}
    \PP_{\theta}[Y\geq m\mid Z=1] = \frac{1+\PP_{\theta}[Y=m\mid Z=1]}{2} \geq  \frac{1}{2}.
\end{align*}
Since the number of pulls of $\cS_0$ can only increase after the exploration phase, we obtain that
\begin{align*}
    q_{\alg}(m)  & \geq  \PP_{\theta}\big[ N_{\cS_0}(T)\geq m
    \big] \\
    & \geq \PP[Z=1] \PP_{\theta}\big[N_{\cS_0}(2m)\geq m \,\big|\, Z=1 \big]  \\
    & \geq  \frac{1}{2} \min\bigg\{  1,\, \frac{K-A}{K}\sqrt{\frac{T}{m}}  \bigg\}.
\end{align*}
Together with~\eqref{eq:tail-envelope-lower-tight-upper}, it is indicated that both bounds in Lemma~\ref{lem:pull-count-tail} cannot be improved up to logarithmic factors.

\subsection{Exploration profile of Near Minimax-Optimal Algorithms}\label{app:exploration-profile-minimax-optimal}

In this section, we provide the formal statement and proof of the exploration lower bound for algorithms that are near-minimax optimal on $\class{K}{B}$, referenced in Section~\ref{sec:non-adaptivity-A-large-overview}. The argument follows the same strategy as Lemma~\ref{lem:pull-count-tail},
but is applied in the opposite direction: instead of showing that an algorithm near-optimal on $\class{K}{A}$ rarely concentrates too many pulls on a fixed small set, we show that an algorithm near-optimal on $\class{K}{B}$ cannot avoid pulling a fixed set of size $B$ sufficiently often.

\begin{lemma}\label{lem:exploration-profile-minimax-optimal}
Let $1\leq B<K$, and suppose that $\alg$ is minimax-optimal on $\class{K}{B}$ up to a factor $\Gamma\geq1$, namely
\begin{align*}
    \sup_{\nu:\,|\aopt(\nu)|=B} \regret_{\nu,\alg}(T) \leq \Gamma\frac{K-B}{\sqrt{KB}}\sqrt{T}.
\end{align*}
Let $\theta$ denote the all-equal instance. Then, for every $\cU\subseteq\cA$ with $|\cU|=B$, every $\ell\geq1$, and every integer $m$ with $1\leq m\leq T/2$,
\begin{align}
    \PP_{\theta}\big[N_{\cU}(T)<m\big] \leq \exp(-\ell) + 16\Gamma\sqrt{\frac{K\ell m}{BT}}. \label{eq:exploration-profile-minimax-optimal-main}
\end{align}
In particular, for any $\delta=\Theta(1)$, taking $\ell=\log(2/\delta)$, sufficient large $T$ and
\begin{align*}
    m_B := \bigg\lfloor \frac{\delta^2BT}{2^{10}\Gamma^2K \ell} \bigg\rfloor = \widetilde{\Omega}\bigg(\frac{TB}{K}\bigg),
\end{align*}
we obtain $\PP_{\theta}\big[N_{\cU}(T)\geq m_B\big] \geq 1 - \delta$.
\end{lemma}

\begin{proof}[Proof of Lemma~\ref{lem:exploration-profile-minimax-optimal}]
Fix $\cU\subseteq\cA$ with $|\cU|=B$, $\ell\geq1$, and an integer $m\leq T/2$, and write $\Delta:=(4\sqrt{m\ell})^{-1}\leq1/4$. Consider the instance $\mu$ obtained from the all-equal instance $\theta$ by increasing the mean rewards on $\cU$ by $\Delta$:
\begin{align*}
    \mu_a := \begin{cases}
        1/2+\Delta, & a\in\cU,\\
        1/2, & a\in\cU^{\complement}.
    \end{cases}
\end{align*}
Since $|\cU|=B$, every arm in $\cU$ is strictly better than every arm outside $\cU$, so $\mu$ has exactly $B$ optimal arms, namely $\aopt(\mu)=\cU$, and hence $\mu\in\class{K}{B}$. As in the proof of Lemma~\ref{lem:pull-count-tail}, define
\begin{align*}
    H:=\{N_{\cU}(T)\geq m\}, \qquad \tau:=\inf\{t\in[T]:N_{\cU}(t)=m\}\wedge T,
\end{align*}
so that $\tau$ is a bounded $(\cF_t)_{t=0}^T$-stopping time and, since the pull count $N_{\cU}(t)$ increases by at most one per round, $H=\{N_{\cU}(\tau)=m\}\in\cF_{\tau}$. Consequently $H^{\complement}=\{N_{\cU}(T)<m\}$ also belongs to $\cF_{\tau}$, and on $H^{\complement}$ we have $\tau=T$.

We first show that, up to $\tau$, the trajectory distributions induced by $\theta$ and $\mu$ remain close. Writing $\PP_\theta^{\tau}:=\PP_\theta|_{\cF_{\tau}}$ and $\PP_{\mu}^{\tau}:=\PP_{\mu}|_{\cF_{\tau}}$, and using $\theta_a-\mu_a=-\Delta\ind\{a\in\cU\}$, Proposition~\ref{prop:gaussian-bandit-likelihood-ratio} gives
\begin{align*}
    \log \frac{\ud\PP_\theta^{\tau}}{\ud\PP_{\mu}^{\tau}} & = \sum_{t=1}^{\tau} \bigg[ (\theta_{a_t}-\mu_{a_t}) (r_t-\theta_{a_t}) + \frac{1}{2}(\theta_{a_t}-\mu_{a_t})^2\bigg] \\
    & = -\Delta \sum_{t=1}^{\tau} \ind\{a_t\in\cU\}(r_t-1/2) +
    \frac{\Delta^2}{2}
    N_{\cU}(\tau).
\end{align*}
Denote $M_t:=-\Delta\sum_{s=1}^t\ind\{a_s\in\cU\}(r_s-1/2)$ and $V_t:=\Delta^2N_{\cU}(t)$. Exactly as in the proof of Lemma~\ref{lem:pull-count-tail}, since $r_t-1/2$ is $1$-sub-Gaussian conditioned on $\cG_t$ under $\PP_\theta$ (regardless of the sign of $\Delta$), $\exp(\lambda M_t-\lambda^2V_t/2)$ is a super-martingale under $\PP_\theta$ for every $\lambda\in\RR$. Invoking Lemma~\ref{lem:sub-gaussian-super-martingale} with $x=\Delta\sqrt{2m\ell}$, and using $V_{\tau}\leq m\Delta^2$, we obtain
\begin{align*}
    \PP_{\theta}\big(M_\tau \ge \Delta\sqrt{2m\ell}\big) \leq \exp\bigg(-\frac{2m\Delta^2\ell}{2m\Delta^2}\bigg) = \exp(-\ell).
\end{align*}
Since $\Delta^2N_{\cU}(\tau)/2 \leq m\Delta^2/2$, this gives
\begin{align*}
    \PP_{\theta} \bigg(\log \frac{\ud\PP_\theta^{\tau}}{\ud\PP_{\mu}^{\tau}} \geq \Delta\sqrt{2m\ell} + m\Delta^2/2 \bigg) \leq \exp(-\ell).
\end{align*}
Recalling $\Delta=(4\sqrt{m\ell})^{-1}$, we have $\Delta\sqrt{2m\ell}+m\Delta^2/2 = 1/(2\sqrt2) + 1/(32\ell) < \log2$ since $\ell\geq1$, so
\begin{align}
    \PP_{\theta} \bigg(\log \frac{\ud\PP_\theta^{\tau}}{\ud\PP_{\mu}^{\tau}} > \log 2 \bigg) \leq \exp(-\ell). \label{eq:exploration-profile-minimax-optimal-llr}
\end{align}

We now transfer the bound on $\PP_\theta(H^{\complement})$ to $\mu$. Since $H^{\complement}\in\cF_{\tau}$, we apply Proposition~\ref{prop:change-of-measure} on $(\Omega,\cF_{\tau})$ with $\PP=\PP_\theta^{\tau}$, $\QQ=\PP_{\mu}^{\tau}$, $X=\ind\{H^{\complement}\}$, $b=1$ and $\gamma=\log2$; using~\eqref{eq:exploration-profile-minimax-optimal-llr}, this gives
\begin{align}
    \PP_{\mu}\big(H^{\complement}\big) \geq \frac{1}{2}\Big[\PP_\theta\big(H^{\complement}\big) - \exp(-\ell)\Big]. \label{eq:exploration-profile-minimax-optimal-com}
\end{align}

It remains to bound the left-hand side of~\eqref{eq:exploration-profile-minimax-optimal-com} using the regret guarantee of $\alg$ on $\mu$. Every arm in $\cU^{\complement}$ is suboptimal under $\mu$ with gap $\Delta$, so
\begin{align*}
    \regret_{\mu,\alg}(T) = \Delta\,\EE_{\mu}\big[N_{\cU^{\complement}}(T)\big] = \Delta\big(T-\EE_{\mu}[N_{\cU}(T)]\big).
\end{align*}
On the event $H^{\complement}=\{N_{\cU}(T)<m\}$ we have $N_{\cU^{\complement}}(T)>T-m$, so, using~\eqref{eq:exploration-profile-minimax-optimal-com} and $m\leq T/2$,
\begin{align*}
    \regret_{\mu,\alg}(T) & \geq \Delta(T-m)\PP_{\mu}\big(H^{\complement}\big) \\
    & \geq \frac{\Delta(T-m)}{2}\Big[\PP_\theta\big(H^{\complement}\big)-\exp(-\ell)\Big] \\
    & \geq \frac{\Delta T}{4}\Big[\PP_\theta\big(H^{\complement}\big)-\exp(-\ell)\Big].
\end{align*}
On the other hand, since $\mu\in\class{K}{B}$, the assumed minimax-optimality of $\alg$ on $\class{K}{B}$ gives
\begin{align*}
    \regret_{\mu,\alg}(T) \leq \Gamma\frac{K-B}{\sqrt{KB}}\sqrt{T} \leq \Gamma\sqrt{\frac{K}{B}}\sqrt{T},
\end{align*}
where the second inequality uses $K-B\leq K$. Combining the last two displays and rearranging leads to
\begin{align*}
    \PP_\theta\big(H^{\complement}\big) \leq \exp(-\ell) + \frac{4\Gamma}{\Delta\sqrt{T}}\sqrt{\frac{K}{B}} = \exp(-\ell) + 16\Gamma\sqrt{\frac{K\ell m}{BT}},
\end{align*}
where the equality substitutes $\Delta^{-1}=4\sqrt{m\ell}$. This proves the claimed result in~\eqref{eq:exploration-profile-minimax-optimal-main}.
\end{proof}


\subsection{Proof of Theorem~\ref{thm:non-adaptivity-A-small}}

\begin{proof}[Proof of Theorem~\ref{thm:non-adaptivity-A-small}]

The proof follows a Le Cam two-point argument similar to the proof of Theorem~\ref{thm:lower-few-optimal-arms}. For simplicity, for any subset $\mathcal{S} \subseteq \mathcal{A}$, we define $N_{\mathcal{S}}(T) := \sum_{a \in \mathcal{S}} N_{a}(T)$. Fix an arbitrary subset $\mathcal{S}_0 \subseteq \mathcal{A}$ with $|\mathcal{S}_0|=A$. For a parameter $\Delta \in (0,1/4]$ to be specified later, consider the reference instance $\nu_0$ with reward function
\begin{align*}
    r_0(a) :=
    \begin{cases}
        1/2, & a \in \mathcal{S}_0,\\
        1/2-\Delta, & a \in \mathcal{S}_0^\complement.
    \end{cases}
\end{align*}
Thus, $\nu_0$ has exactly $A$ optimal arms, namely the arms in $\mathcal{S}_0$. Let $\mathbb{P}_0$ and $\mathbb{E}_0$ denote probability and expectation under $\nu_0$ and $\mathsf{Alg}$. By the assumed regret guarantee of $\mathsf{Alg}$ on every instance with $A$ optimal arms,
\begin{align*}
    \Delta \mathbb{E}_0[N_{\mathcal{S}_0^\complement}(T)]
    =
    \operatorname{Regret}_{\nu_0,\mathsf{Alg}}(T)
    \leq R_A := \Gamma_{K,T}\frac{K-A}{\sqrt{KA}}\sqrt{T},
\end{align*}
which gives that $\mathbb{E}_0[N_{\mathcal{S}_0^\complement}(T)] \leq R_A/\Delta$.

Since $B \leq K-A$, we can choose a subset $\mathcal{S}_1 \subseteq \mathcal{S}_0^\complement$ with $|\mathcal{S}_1|=B$ consisting of the $B$ least-sampled arms in $\mathcal{S}_0^\complement$. By averaging,
\begin{align}
    \mathbb{E}_0[N_{\mathcal{S}_1}(T)]
    \leq
    \frac{B}{K-A}
    \mathbb{E}_0[N_{\mathcal{S}_0^\complement}(T)]
    \leq
    \frac{BR_A}{(K-A)\Delta},
    \label{eq:nonadapt-lightly-sampled-B}
\end{align}
where the first inequality holds due to that $\mathcal{S}_1$ is the least sampled set under $\nu_0$. We now construct the alternative instance $\nu_1$ by promoting the arms in $\mathcal{S}_1$:
\begin{align*}
    r_1(a) :=
    \begin{cases}
        1/2, & a \in \mathcal{S}_0,\\
        1/2+\Delta, & a \in \mathcal{S}_1,\\
        1/2-\Delta, & a \in (\mathcal{S}_0 \cup \mathcal{S}_1)^\complement.
    \end{cases}
\end{align*}
The instance $\nu_1$ has exactly $B$ optimal arms, namely the arms in $\mathcal{S}_1$. Let $\mathbb{P}_1$ and $\mathbb{E}_1$ denote probability and expectation under $\nu_1$ and $\mathsf{Alg}$.

We next apply a change-of-measure argument. Since $N_{\mathcal{S}_1}(T) \in [0,T]$, Lemma~\ref{lem:tv-change-of-measure} gives
\begin{align}
    \mathbb{E}_1[N_{\mathcal{S}_1}(T)] \leq     \mathbb{E}_0[N_{\mathcal{S}_1}(T)] + T \operatorname{TV}(\mathbb{P}_0,\mathbb{P}_1) \leq \mathbb{E}_0[N_{\mathcal{S}_1}(T)] + T \sqrt{\frac{1}{2}\operatorname{KL}(\mathbb{P}_0\|\mathbb{P}_1)}, \label{eq:nonadapt-TV-expectation}
\end{align}
where the last inequality holds due to Pinsker's inequality (Lemma~\ref{lem:pinsker}). To control the KL-divergence, we see that the instances $\nu_0$ and $\nu_1$ differ only on the arms in $\mathcal{S}_1$. For every $a \in \mathcal{S}_1$, the reward distribution is $\mathcal{N}(1/2-\Delta,1)$ under $\nu_0$ and $\mathcal{N}(1/2+\Delta,1)$ under $\nu_1$. Hence, by the KL decomposition for adaptive bandit experiments,
\begin{align*}
    \operatorname{KL}(\mathbb{P}_0\|\mathbb{P}_1)
    &=
    \sum_{a \in \mathcal{S}_1}
    \mathbb{E}_0[N_{a}(T)]
    \operatorname{KL}\big(
        \mathcal{N}(1/2-\Delta,1)
        \|
        \mathcal{N}(1/2+\Delta,1)
    \big)\\
    &=
    2\Delta^2
    \mathbb{E}_0[N_{\mathcal{S}_1}(T)]\\
    &\leq
    \frac{2\Delta BR_A}{K-A},
\end{align*}
where the last inequality follows from~\eqref{eq:nonadapt-lightly-sampled-B}.

Now choose $\Delta := \min\{1/4,\,(K-A)/(16BR_A)\}$. Then $\operatorname{KL}(\mathbb{P}_0\|\mathbb{P}_1) \leq 1/8$, and consequently we know that $\operatorname{TV}(\mathbb{P}_0,\mathbb{P}_1) \leq 1/4$. Moreover, by~\eqref{eq:nonadapt-lightly-sampled-B},
\begin{align*}
    \mathbb{E}_0[N_{\mathcal{S}_1}(T)] \leq
    \frac{BR_A}{K-A}
    \max\bigg\{
        4,
        \frac{16BR_A}{K-A}
    \bigg\} =
    \max\bigg\{
        \frac{4BR_A}{K-A},
        \frac{16B^2R_A^2}{(K-A)^2}
    \bigg\} \leq
    \frac{T}{4},
\end{align*}
where the last inequality follows from the condition
\begin{align*}
    T \geq \max\{16BR_A/(K-A),\,64B^2R_A^2/(K-A)^2\},
\end{align*}
which can be derived directly from $B \leq A/(8\Gamma_{K,T})$ and $T \geq 256B^2\Gamma^2_{K,T}/(KA)$. Now we combine this inequality with~\eqref{eq:nonadapt-TV-expectation} and $\operatorname{TV}(\mathbb{P}_0,\mathbb{P}_1) \leq 1/4$, which leads to
\begin{align*}
    \mathbb{E}_1[N_{\mathcal{S}_1}(T)]
    \leq
    \frac{T}{4}
    +
    \frac{T}{4}
    =
    \frac{T}{2}.
\end{align*}
Finally, under $\nu_1$, every arm outside $\mathcal{S}_1$ has suboptimality gap at least $\Delta$. Therefore,
\begin{align*}
    \operatorname{Regret}_{\nu_1,\mathsf{Alg}}(T)
    \geq  \Delta   \mathbb{E}_1[T-N_{\mathcal{S}_1}(T)] \geq  \frac{\Delta T}{2}
    =  \min\bigg\{\frac{T}{8}, \textcolor{red}{ \frac{(K-A)T}{32BR_A} } \bigg\}.
\end{align*}
Recall the condition that $T \geq KA(16B^2\Gamma_{K,T}^2)^{-1}$, we have that $T/8 \geq (K-A)T/(32BR_A)$, leading to 
\begin{align*}
    \regret_{\nu_1,\mathsf{Alg}}(T) \geq \frac{(K-A)T}{32BR_A} = \frac{\sqrt{KA}}{32B\Gamma_{K,T}}\sqrt{T}.
\end{align*}
Since $\nu_1$ has exactly $B$ optimal arms, this proves the theorem.
\end{proof}

\subsection{Proof of Theorem~\ref{thm:non-adaptivity-A-large}}

In this section, we provide the proof of Theorem~\ref{thm:non-adaptivity-A-large}. We first need the following technical result.

\begin{lemma}[Aggregation of square-root tails]
\label{lem:aggregation-square-root-tail}
Let $X_1,\ldots,X_q$ be nonnegative integer-valued random variables, not necessarily independent. Suppose that, for some $\delta,c\geq 0$, each $X_i$ satisfies
\begin{align*}
    \mathbb{P}[X_i \geq s]
    \leq
    \delta + \frac{c}{\sqrt{s}}
\end{align*}
for every $i\in[q]$ and every integer $s\geq 1$. Then, for any integer $m\geq 1$,
\begin{align*}
    \mathbb{P}\bigg[\sum_{i=1}^q X_i > m\bigg]
    \leq
    2q\delta + \frac{3qc}{\sqrt{m}}.
\end{align*}
\end{lemma}

\begin{proof}
For each $i\in[q]$, define the truncated random variable $\tilde{X}_i := X_i \wedge m$. We decompose the event $\{\sum_{i=1}^q X_i > m\}$ according to whether any individual random variable exceeds the truncation level $m$. In particular,
\begin{align*}
    \bigg\{\sum_{i=1}^q X_i > m\bigg\}
    \subseteq
    \bigg\{\max_{i\in[q]}X_i > m\bigg\}
    \cup
    \bigg\{\sum_{i=1}^q \widetilde{X}_i > m\bigg\}.
\end{align*}
For the first event, by a union bound and the assumed tail bound,
\begin{align*}
    \mathbb{P}\bigg[\max_{i\in[q]}X_i > m\bigg] \leq
    \sum_{i=1}^q \mathbb{P}[X_i > m] \leq
    q\delta + \frac{qc}{\sqrt{m}}.
\end{align*}
For the second event, since $X_i$ is nonnegative and integer-valued, the tail-sum formula gives
\begin{align*}
    \mathbb{E}[\widetilde{X}_i] = \sum_{s=1}^m \mathbb{P}[X_i\geq s] \leq  m\delta + c\sum_{s=1}^m\frac{1}{\sqrt{s}} \leq  m\delta + 2c\sqrt{m},
\end{align*}
where the last inequality follows from $\sum_{s=1}^m 1/\sqrt{s} \leq 2\sqrt{m}$. Therefore, by Markov's inequality,
\begin{align*}
    \mathbb{P}\bigg[\sum_{i=1}^q \widetilde{X}_i > m\bigg] \leq  \frac{1}{m}  \sum_{i=1}^q \mathbb{E} [\widetilde{X}_i] \leq q\delta + \frac{2qc}{\sqrt{m}}.
\end{align*}
Combining the two bounds yields
\begin{align*}
    \mathbb{P}\bigg[\sum_{i=1}^q X_i > m\bigg]
    \leq
    2q\delta + \frac{3qc}{\sqrt{m}},
\end{align*}
which finishes the proof.
\end{proof}

Now we are ready to prove Theorem~\ref{thm:non-adaptivity-A-large}. 

\begin{proof}[Proof of Theorem~\ref{thm:non-adaptivity-A-large}]
For convenience, define $R_A := \Gamma_{K,T} (K-A)\sqrt{T}/K$, which is the near-minimax optimal rate on $\class{K}{A}$. Fix an arbitrary subset $\mathcal{U}\subseteq\mathcal{A}$ with $|\mathcal{U}|=B$. For a parameter $\Delta>0$ to be specified later, consider the bandit instance $\nu$ with reward function
\begin{align*}
    r_\nu(a)
    :=
    \begin{cases}
        1/2+\Delta, & a\in\mathcal{U},\\
        1/2, & a\in\mathcal{U}^{\complement}.
    \end{cases}
\end{align*}
Thus, $\nu$ has exactly $B$ optimal arms, namely the arms in $\mathcal{U}$. Recall that $\theta$ denotes the all-equal reference instance with $r_\theta(a)=1/2$ for every $a\in\mathcal{A}$. We use $\mathbb{P}_\nu$ and $\mathbb{P}_\theta$ to denote the trajectory distributions induced by $\mathsf{Alg}$ under $\nu$ and $\theta$, respectively.

We first show that, under the reference instance $\theta$, the algorithm samples $\mathcal{U}$ only $\widetilde{O}(q^2R_A^2)$ times with constant probability, where
$q:=\lceil B/(K-A)\rceil$. Partition $\mathcal{U}$ into $q$ disjoint subsets
$\mathcal{U}_1,\ldots,\mathcal{U}_q$ such that $|\mathcal{U}_i|\leq K-A$ for every $i\in[q]$. For each $i\in[q]$, extend $\mathcal{U}_i$ to a subset $\mathcal{S}_i\subseteq\mathcal{A}$ satisfying $\mathcal{U}_i\subseteq\mathcal{S}_i$ and $|\mathcal{S}_i|=K-A$. Then
\begin{align*}
    N_{\mathcal{U}}(T)
    =
    \sum_{i=1}^q N_{\mathcal{U}_i}(T)
    \leq
    \sum_{i=1}^q N_{\mathcal{S}_i}(T).
\end{align*}
Set $\ell:=\log(16q)$. Then, for every $s\in[T]$ and every $i\in[q]$, applying Lemma~\ref{lem:pull-count-tail} gives
\begin{align*}
    \mathbb{P}_\theta
    \big[N_{\mathcal{S}_i}(T)\geq s\big]
    \leq
    \frac{1}{16q}
    +
    8R_A\sqrt{\frac{\ell}{s}}.
\end{align*}
Since $N_{\mathcal{S}_i}(T)\leq T$, the same inequality holds trivially for $s>T$. Therefore, applying Lemma~\ref{lem:aggregation-square-root-tail} to
$X_i:=N_{\mathcal{S}_i}(T)$ with $\delta=1/(16q)$ and
$c=8R_A\sqrt{\ell}$ gives, for every $m\geq1$,
\begin{align*}
    \mathbb{P}_\theta
    \bigg[
        \sum_{i=1}^qN_{\mathcal{S}_i}(T)>m
    \bigg]
    \leq
    \frac18
    +
    24qR_A\sqrt{\frac{\ell}{m}}.
\end{align*}

Now we choose $m=\lceil 4096q^2R_A^2\ell \rceil$. Then $24qR_A\sqrt{\ell/m}\leq3/8$, which leads to 
\begin{align*}
    \mathbb{P}_\theta \bigg[ \sum_{i=1}^qN_{\mathcal{S}_i}(T)>m \bigg] \leq  \frac12.
\end{align*}
Since $N_{\mathcal{U}}(T)\leq\sum_{i=1}^qN_{\mathcal{S}_i}(T)$, we conclude that
\begin{align}
    \mathbb{P}_\theta
    \big[N_{\mathcal{U}}(T)\leq m\big]
    \geq
    \frac12.
    \label{eq:dense-to-sparse-typical-pulls}
\end{align}
We now show that $T\geq2m$. Since $B\geq K-A$, we have $q = \lceil B/(K-A) \rceil \leq 2B/(K-A)$. Consequently, $\ell \leq \log(32B/(K-A))$. Using the definition of $R_A$,
\begin{align*}
    4096q^2R_A^2\ell & \leq 16384 \frac{B^2}{(K-A)^2}  \Gamma_{K,T}^2 \frac{(K-A)^2T}{K^2} \log\bigg(\frac{32B}{K-A}\bigg) \\
    & = 16384\frac{       B^2\Gamma_{K,T}^2 \log\big(32B/(K-A)\big) }{K^2} T.
\end{align*}
By the assumed condition
\begin{align*}
    B\Gamma_{K,T}\sqrt{ \log\bigg(\frac{32B}{K-A}\bigg)} \leq \frac{K}{256},
\end{align*}
the right-hand side is at most $T/4$. Therefore,
\begin{align*}
    m =\lceil 4096q^2R_A^2\ell \rceil \leq 4096q^2R_A^2\ell + 1 \leq \frac{T}{4}+1 \leq     \frac{T}{2},
\end{align*}
where the last inequality follows from $T\geq4$.

We now transfer~\eqref{eq:dense-to-sparse-typical-pulls} from the all-equal reference instance $\theta$ to the hard instance $\nu$. Once again, we consider the stopping time
\begin{align*}
    \mathcal{E} := \big\{N_{\mathcal{U}}(T)\leq m\big\}, \qquad \tau:=     \inf\big\{t\in[T]:N_{\mathcal{U}}(t)=m+1\big\}
    \wedge T.
\end{align*}
so that $\mathbb{P}_\theta[\mathcal{E}]\geq1/2$. Notice that $\mathcal{E}\in\mathcal{F}_\tau$ and $N_{\mathcal{U}}(\tau)\leq m+1$ almost surely. Since $\theta$ and $\nu$ differ only on the arms in $\mathcal{U}$, where the reward distribution is $\mathcal{N}(1/2,1)$ under $\theta$ and $\mathcal{N}(1/2+\Delta,1)$ under $\nu$, the KL decomposition for the stopped trajectory gives
\begin{align*}
    \KL\big(
        \mathbb{P}_\theta|_{\mathcal{F}_\tau}
        \|
        \mathbb{P}_\nu|_{\mathcal{F}_\tau}
    \big)
    =
    \frac{\Delta^2}{2}
    \mathbb{E}_\theta[N_{\mathcal{U}}(\tau)]
    \leq
    \frac{\Delta^2(m+1)}{2}.
\end{align*}
Now we choose $\Delta:=1/(4\sqrt{m+1})$. Then we see that $\kl{\PP_{\theta}|_{\cF_\tau}}{\PP_{\nu}|_{\cF_\tau}} \leq 1/32$. By Pinsker's inequality (Lemma~\ref{lem:pinsker}), $\operatorname{TV(\PP_{\theta|}|_{\cF_\tau}}, \PP_{\nu}|_{\cF_\tau}) \leq 1/8$. Since $\mathcal{E}\in\mathcal{F}_\tau$, it follows from~\eqref{eq:dense-to-sparse-typical-pulls} that
\begin{align*}
    \mathbb{P}_\nu[\mathcal{E}] \geq \mathbb{P}_\theta[\mathcal{E}] -    \operatorname{TV}\big(\mathbb{P}_\theta|_{\mathcal{F}_\tau},      \mathbb{P}_\nu|_{\mathcal{F}_\tau} \big) \geq \frac{3}{8}.
\end{align*}

Recall that the arms in $\mathcal{U}$ are the unique optimal arms under $\nu$, while every arm in $\mathcal{U}^{\complement}$ has suboptimality gap $\Delta$. Therefore,
\begin{align*}
    \regret_{\nu,\mathsf{Alg}}(T) =  \Delta \mathbb{E}_\nu
    \big[
        T-N_{\mathcal{U}}(T)
    \big] \geq   \Delta(T-m) \mathbb{P}_\nu[\mathcal{E}] \geq   \frac{3\Delta T}{16} =     \frac{3T}{64\sqrt{m+1}},
\end{align*}
where the last inequality holds due to $m\leq T/2$.

It remains to simplify this lower bound. Define $\bar{\ell} = \log(32B/(K-A))$. Since $B/(K-A)\geq1$, we have $\ell=\log(16q)\geq\log(16B/(K-A))\geq(4/5)\bar{\ell}$. Together with the assumed horizon condition $T \geq K^2/(B^2\Gamma_{K,T}^2\bar{\ell})$ gives
\begin{align*}
    q^2R_A^2\ell \geq \frac{B^2}{(K-A)^2} \Gamma_{K,T}^2 \frac{(K-A)^2T}{K^2} \ell \geq \frac{4}{5} \frac{B^2\Gamma_{K,T}^2T\bar{\ell}}{K^2}  \geq\frac45.
\end{align*}
Consequently, $m + 1 \leq 4096 q^2R_A^2\ell +2 \leq 8192 q^2R_A^2\ell$. Therefore,
\begin{align*}
    \regret_{\nu,\mathsf{Alg}}(T) \geq \frac{3T}{64\sqrt{m+1}} \geq     \frac{3}{8192\sqrt{2}} \frac{\sqrt{KA}}{B\Gamma_{K,T}\sqrt{\log(32B/(K-A))}}\sqrt{T},
\end{align*}
where the second inequality uses $q\leq2B/(K-A)$, $K \geq A$ and $\ell\leq\log(32B/(K-A))$. This finishes the proof.
\end{proof}

\section{High-probability Lower Bound}
\label{app:high-prob}
As discussed in Remark~\ref{rmk:high-prob}, one may ask whether the $\tilde{O}((K-A)\sqrt{T}/K)$ rate in Theorem \ref{thm:upper-bound} can be extended to the high-probability setting. 
In this section, we will show that it is impossible and there is a necessary payoff  to achieve the near-optimal expected rate. In particular, let $\{a_t\}_{t=1}^T$ be the actions selected by $\alg$. We write
\begin{align*}
    R_T = \sum_{t=1}^T r^* - \sum_{t=1}^T r(a_t).
\end{align*}
As a result, we have $\EE[R_T] = \regret(T)$. Our main hardness result is summarized in the following theorem:

\begin{theorem}\label{thm:high-prob-lower}
Suppose $A>K/2$, so that $(K-A)/K < 1/2$, and fix $\Gamma\ge1$. Define $\ell = \log(16K/(K-A))$ and $c_0 = 13/128$. Suppose that $\alg$ is a strategy satisfying for any problem instance in $\class{K}{A}$,
\begin{align*}
    \regret(T)\le\Gamma\frac{K-A}{K}\sqrt T,
\end{align*}
that is, a policy that achieves the minimax-optimal expected regret up to a
factor of $\Gamma$.
If $T\ge {8\Gamma^2\ell}/{c_0^2}$, then there exists a problem instance $\mu\in\class{K}{A}$ such that
\begin{align}
    \PP_{\mu,\alg}\bigg(R(T) \ge \frac{c_0}{8\sqrt2} \frac{\sqrt T}{\Gamma\ell} \bigg)  \ge\frac{K-A}{32K}. \label{eq:high-prob-lower-main}
\end{align}
\end{theorem}

Theorem~\ref{thm:high-prob-lower} indicates that, for any near-minimax optimal algorithm $\alg$, there exists some instance, such that the pseudo-regret of $\alg$ reaches $\tilde\Omega(\sqrt T)$, with probability $\Omega((K-A)/K)$. 

\begin{remark}
    It is widely known that UCB-typed algorithm, while achieving a near-minimax-optimal rate, simultaneously achieves a $\tilde{O}(\sqrt{KT}\text{polylog}(\delta^{-1}))$ high-probability upper bound. Theorem~\ref{thm:high-prob-lower} shows that, the analogue rate, $\tilde{O}\big((K-A)K^{-1} \sqrt{T} \cdot \mathrm{polylog}(\delta^{-1})\big)$ is impossible when $A = K - o(K)$. In particular, making $\tilde{O}\big((K-A)K^{-1} \sqrt{T} \cdot \mathrm{polylog}(\delta^{-1})\big) = \tilde{O}(\sqrt{T})$ requires $\delta \sim \exp(-K/(K-A)) \ll (K-A)/K$, which contradicts~\eqref{eq:high-prob-lower-main}.
\end{remark}

\begin{proof}[Proof of Theorem~\ref{thm:high-prob-lower}]
Recall that $\ell = \log(16K/(K-A))$ and $c_0 = 13/128$. Since $T\geq 8\Gamma^2\ell/c_0^2$, we can choose an integer $m$ satisfying
\begin{align}
    \frac{c_0^2T}{8\Gamma^2\ell}
    \leq m
    \leq
    \frac{c_0^2T}{4\Gamma^2\ell}.
    \label{eq:high-prob-lower-choice-m}
\end{align}
Let $\theta$ be the instance such that $\btheta=\mathbf{1}/2\in\RR^K$ denote the all-equal auxiliary instance. We first apply Lemma~\ref{lem:pull-count-tail} to characterize the exploration behavior of $\alg$ under $\theta$. By the lower bound in Lemma~\ref{lem:pull-count-tail},
\begin{align*}
    \max_{\cS\subseteq[K]:\,|\cS| = K-A} \PP_{\theta}\big[N_{\cS}(T) \geq m\big] & \geq \frac{K-A}{K} - \exp(-\ell) - 8\Gamma\frac{K-A}{K}\sqrt{\frac{m\ell}{T}} \\
    & \geq  \frac{K-A}{K} - \frac{K-A}{16K} - \frac{13(K-A)}{32K} \\
    &\geq \frac{K-A}{2K},
\end{align*}
where the second inequality follows from $\exp(-\ell)=(K-A)/(16K)$ and the upper bound on $m$ in~\eqref{eq:high-prob-lower-choice-m}. Therefore, there exists a subset $\cS\subseteq[K]$ with $|\cS|=K-A$ such that
\begin{align}
    \PP_{\theta}\big[N_{\cS}(T)\geq m\big]
    \geq
    \frac{K-A}{2K}.
    \label{eq:high-prob-lower-reference-tail}
\end{align}

We now construct the hard instance $\mu$ by decreasing the mean rewards of the arms in $\cS$ by $\Delta$, with $\Delta = (4\sqrt{m\ell})^{-1}$, namely
\begin{align*}
    \mu_a :=  \begin{cases}
        \frac{1}{2} - \Delta, & a\in\cS,\\
        \frac{1}{2}, & a\in\cS^\complement.
    \end{cases}
\end{align*}
Since $|\cS|=K-A$, the instance $\mu$ has exactly $A$ optimal arms and hence belongs to $\class{K}{A}$.
Now we transfer~\eqref{eq:high-prob-lower-reference-tail} from the auxiliary instance $\theta$ to $\mu$. As usual, we prove that at $\tau$, the probability distributions induced by $\theta$ and $\mu$ are still close, therefore leading to that $\PP_{\mu} [H] \gtrsim (K-A)/K$, where $H$ and $\tau$ are defined by
\begin{align*}
    H:=\big\{N_{\cS}(T)\geq m\big\},
    \qquad
    \tau:=\inf\big\{t\in[T]:N_{\cS}(t)=m\big\}\wedge T.
\end{align*}
In particular, consider the
trajectory distributions restricted to $\cF_{\tau}$ by $\PP_\theta^{\tau} := \PP_\theta|_{\cF_{\tau}}$ and $\PP_{\mu}^{\tau} := \PP_{\mu}|_{\cF_{\tau}}$. Applying Proposition~\ref{prop:gaussian-bandit-likelihood-ratio} gives that
\begin{align*}
    \log \frac{\ud\PP_\theta|_{\cF_{\tau}}}{\ud\PP_\mu|_{\cF_{\tau}}} & = \sum_{t=1}^{\tau} \bigg[ (\theta_{a_t}-\mu_{a_t}) (r_t-\theta_{a_t}) + \frac{1}{2}(\theta_{a_t}-\mu_{a_t})^2\bigg] \\
    & = \Delta \sum_{t=1}^\tau \ind\{a_t\in\cS\}(r_t-1/2) +
    \frac{\Delta^2}{2}
    N_{\cS}(\tau).
\end{align*}
Now we denote $M_t = \Delta \sum_{s=1}^t \ind\{a_s\in\cS\}(r_s-1/2)$ and $V_t = \Delta^2 N_{\cS}(t)$, then the probability distribution $\PP_\theta$, we see that for any $\lambda \in \RR$,
\begin{align*}
    & \EE_{\theta}\bigg[\exp\bigg(\lambda M_t-\frac{\lambda^2V_t}{2} \bigg) \bigg| \cF_{t-1}\bigg] \\
    & \quad = \exp\bigg(\lambda M_{t-1}-\frac{\lambda^2V_{t-1}}{2} \bigg) \EE_{\theta}\bigg[\underbrace{\exp\bigg(\lambda (r_t-1/2) \Delta\ind\{a_t\in\cS\}-\frac{\lambda^2\Delta^2 \ind\{a_t\in\cS\}}{2} \bigg)}_{(\star)} \bigg| \cF_{t-1}\bigg] 
\end{align*}
For $(\star)$, we first take expectation on $\cG_t$, which gives that 
\begin{align*}
    \EE_{\theta}[(\star) | \cG_t] &  =1 - \ind\{a_t\in\cS\} +  \ind\{a_t\in\cS\} \EE_{\theta}\bigg[\exp\bigg(\lambda (r_t-1/2) \Delta -\frac{\lambda^2\Delta^2 }{2} \bigg) \bigg| \cG_t \bigg] \\
    & \leq 1 - \ind\{a_t\in\cS\} + \ind\{a_t\in\cS\} \\
    & = 1,
\end{align*}
where the inequality holds due to that $r_t-1/2$ is $1$-sub-gaussian conditioned on $\cG_t$ on $\PP_\theta$. This immediately gives $\EE_{\theta}[(\star) | \cF_{t-1}] \leq 1$, verifying that $\exp(\lambda M_t - \lambda^2 V_t /2)$ is a super-martingale for any $\lambda$. Now invoking Lemma~\ref{lem:sub-gaussian-super-martingale} with $x = \Delta\sqrt{2m\ell}$, since $V_\tau \leq m\Delta^2$, we have
\begin{align*}
    \PP_{\theta}\big(M_\tau \ge \Delta\sqrt{2m\ell}\big) \leq \exp\bigg(-\frac{2m\Delta^2\ell}{2m\Delta^2}\bigg) = \exp(-\ell).
\end{align*}
Since $\Delta^2 N_{\cS}(\tau_{\cS}) /2 \leq m\Delta^2/2$, this immediately leads to
\begin{align*}
    \PP_{\theta} \bigg(\log \frac{\ud\PP_\theta^{\tau}}{\ud\PP_{\mu}^{\tau}} \geq \Delta\sqrt{2m\ell} + m\Delta^2/2 \bigg) \leq \exp(-\ell).
\end{align*}
Recall that $\Delta = (4\sqrt{m\ell})^{-1}$, this means that $\Delta\sqrt{2m\ell} + m\Delta^2/2 = 1/2\sqrt{2} + 1/32\ell < \log 2$ since $\ell \geq 1$. Therefore,
\begin{align*}
    \PP \bigg(\log \frac{\ud\PP_\theta^{\tau}}{\ud\PP_{\mu}^{\tau}} > \log 2 \bigg) \leq \exp(-\ell).
\end{align*}
Now we apply Proposition~\ref{prop:change-of-measure} on the measurable space $(\Omega,\cF_{\tau_{\cS}})$ with $\PP= \PP_\theta^{\tau}$, $\QQ = \PP_{\mu}^{\tau}$, $X = \ind\{H\}$, $b=1$ and $\gamma = \log 2$ and obtain that
\begin{align*}
    \PP_{\mu}(H) = \PP_{\mu}^{\tau}(H) \geq \frac{1}{2} \bigg[\PP_\theta^{\tau}(H) - \PP \bigg(\log \frac{\ud\PP_\theta^{\tau}}{\ud\PP_{\mu}^{\tau}} > \log 2 \bigg) \bigg] \geq \frac{1}{2} \big[ \PP_\theta(H)  -\exp(-\ell) \big],
\end{align*}
where the equation holds due to that $H \in \cF_{\tau}$. 

Finally, every pull of an arm in $\cS$ incurs regret $\Delta$ under $\mu$. Hence, on the event $H$,
\begin{align*}
    R(T)
    \geq
    \Delta m
    =
    \frac{\sqrt{m}}{4\sqrt{\ell}}.
\end{align*}
Using the lower bound on $m$ in~\eqref{eq:high-prob-lower-choice-m}, we obtain
\begin{align*}
    \Delta m
    \geq
    \frac{1}{4\sqrt{\ell}}
    \sqrt{\frac{c_0^2T}{8\Gamma^2\ell}}
    =
    \frac{c_0}{8\sqrt{2}}
    \frac{\sqrt{T}}{\Gamma\ell}.
\end{align*}
Therefore,
\begin{align*}
    \PP_{\mu}\bigg(
        R(T)
        \geq
        \frac{c_0}{8\sqrt{2}}
        \frac{\sqrt{T}}{\Gamma\ell}
    \bigg)
    \geq
    \frac{K-A}{32K},
\end{align*}
which proves the desired result.
\end{proof}

\section{Auxiliary Lemmas}\label{app:auxiliary}

The following standard stopped martingale inequality is a direct
consequence of the time-uniform Chernoff bound of
\citet[Theorem~1]{howard2020time}.

\begin{lemma}[Stopped sub-Gaussian martingale inequality]
\label{lem:sub-gaussian-super-martingale}
Let $(M_t,V_t)_{t=0}^T$ be an adapted process with
$M_0=V_0=0$ and $(V_t)_{t=0}^T$ nondecreasing. Suppose that, for
every $s>0$, $\exp(sM_t - s^2V_t/2)$ is a nonnegative supermartingale. Then, for any stopping time  $\tau\leq T$ and every $x,v>0$,
\begin{align}
    \PP\big(
        M_\tau\geq x,\,
        V_\tau\leq v
    \big)
    \leq
    \exp\left(
        -\frac{x^2}{2v}
    \right).
    \label{eq:stopped-sub-gaussian-martingale}
\end{align}
In particular, if $V_\tau\leq v$ almost surely, then
\begin{align*}
    \PP(M_\tau\geq x)
    \leq
    \exp\left(
        -\frac{x^2}{2v}
    \right).
\end{align*}
\end{lemma}

The following lemma characterizes the change-of-measure under total variation, which is the core components in standard lower bound techniques like Le-Cam's method~\citep{yu1997assouad}. We refer to~\citet[Lemma~2.1]{tsybakov2009introduction} for the proof.

\begin{lemma}\label{lem:tv-change-of-measure}
Let $(\Omega,\mathcal{F})$ be a measurable space, and let
$\mathbb{P}$ and $\mathbb{Q}$ be probability measures on it. Then, for
every $\mathcal{F}$-measurable function $f:\Omega\to[0,T]$, 
\begin{align*}
    \big| \EE_{\PP}[f] - \EE_{\QQ}[f] \big| \leq  T\operatorname{TV}(\mathbb{P},\mathbb{Q}),
\end{align*}
where $\operatorname{TV}(\mathbb{P},\mathbb{Q}) :=  \sup_{E\in\mathcal{F}} \big| \mathbb{P}(E)-\mathbb{Q}(E) \big|$.
\end{lemma}

\begin{lemma}[Pinsker's inequality]
\label{lem:pinsker}
Let $(\Omega,\mathcal{F})$ be a measurable space, and let
$\mathbb{P}$ and $\mathbb{Q}$ be probability measures on it. Then
\begin{align*}
    \operatorname{TV}(\mathbb{P},\mathbb{Q}) \leq \sqrt{ \frac{1}{2}   \operatorname{KL}(\mathbb{P}\|\mathbb{Q})}.
\end{align*}
\end{lemma}

The following lemma is standard for proof of minimax lower bound and we refer to~\citet[Theorem~14.2]{lattimore2020bandit} for the proof.

\begin{lemma}[Bretagnolle--Huber inequality]
\label{lem:bretagnolle-huber}
Let $(\Omega,\mathcal{F})$ be a measurable space, and let
$\mathbb{P}$ and $\mathbb{Q}$ be probability measures on it. Then
for every event $E\in\mathcal{F}$,
\begin{align*}
    \mathbb{P}(E) + \mathbb{Q}(E^{\complement}) \geq \frac{1}{2} \exp\big( -\operatorname{KL}(\mathbb{P}\|\mathbb{Q})\big).
\end{align*}
\end{lemma}

\bibliographystyle{ims}
\bibliography{reference}

\end{document}